\documentclass{article}

\usepackage[final]{corl_2019} 

\usepackage{times}
\usepackage[numbers]{natbib}
\usepackage{amsmath,amssymb}
\usepackage{tensor}
\usepackage{graphicx}
\usepackage{subcaption}
\usepackage{overpic}
\usepackage{rotating}
\usepackage{verbatim}
\graphicspath{{./fig/}}
\usepackage{algorithm}
\usepackage{algorithmic}
\usepackage{color}
\usepackage{soul}
\usepackage{bm}
\usepackage{multirow}
\usepackage{multicol}
\usepackage{arydshln}
\usepackage{xspace}
\usepackage{float}
\usepackage{url}
\usepackage{authblk}
\usepackage{fixltx2e}

\def\flow{RMPflow\xspace}
\def\tree{RMP-tree\xspace}
\def\algebra{RMP-algebra\xspace}
\def\pushforward{ \texttt{\small pushforward}\xspace}
\def\pullback{\texttt{\small pullback}\xspace}
\def\resolve{\texttt{\small resolve}\xspace}

\def\newstuff{*}
\def\newflow{RMPfusion\xspace}
\def\newtree{RMP-tree\newstuff\xspace}
\def\newalgebra{RMP-algebra\newstuff\xspace}
\def\newpullback{\texttt{\small pullback}\newstuff\xspace}

\usepackage{amsmath}
\usepackage{amsfonts}
\usepackage{amssymb}
\usepackage{amsthm}
\usepackage{bm}
\usepackage{bbm}
\usepackage{mathtools}
\usepackage{enumitem}
\usepackage{thmtools,thm-restate}
\usepackage{algorithm}
\usepackage{algorithmic}
\usepackage{natbib}
\usepackage[capitalise]{cleveref}
\usepackage{color}
\usepackage[dvipsnames]{xcolor}
\usepackage{graphicx}
\usepackage{comment}

\theoremstyle{plain}
\newtheorem{lemma}{Lemma}
\newtheorem{theorem}{Theorem}

\theoremstyle{definition}
\newtheorem{definition}{Definition}

\theoremstyle{remark}

\def\CC{\mathcal{C}}

\def\MM{\mathcal{M}}\def\NN{\mathcal{N}}

\def\SS{\mathcal{S}}\def\TT{\mathcal{T}}

\def\Bb{\mathbf{B}}\def\Cb{\mathbf{C}}

\def\Gb{\mathbf{G}}\def\Hb{\mathbf{H}}
\def\Jb{\mathbf{J}}\def\Kb{\mathbf{K}}
\def\Mb{\mathbf{M}}

\def\ab{\mathbf{a}}
\def\fb{\mathbf{f}}
\def\gb{\mathbf{g}}\def\hb{\mathbf{h}}

\def\mbb{\mathbf{m}}

\def\xb{\mathbf{x}}

\def\Rbb{\mathbb{R}}

\def\R{\Rbb}

\def\t{\top}
\def\*{\star}

\usepackage[dvipsnames]{xcolor}

\newcommand{\q}{\mathbf{q}}
\newcommand{\qd}{{\dot{\q}}}
\newcommand{\qdd}{{\ddot{\q}}}

\newcommand{\x}{\mathbf{x}}
\newcommand{\xd}{{\dot{\x}}}
\newcommand{\xdd}{{\ddot{\x}}}
\newcommand{\y}{\mathbf{y}}
\newcommand{\yd}{{\dot{\y}}}

\newcommand{\z}{\mathbf{z}}
\newcommand{\zd}{{\dot{\z}}}

\newcommand{\f}{\mathbf{f}}
\newcommand{\h}{\mathbf{h}}

\newcommand{\J}{\mathbf{J}}

\newcommand{\B}{\mathbf{B}}
\newcommand{\C}{\mathbf{C}}

\newcommand{\G}{\mathbf{G}}

\newcommand{\I}{\mathbf{I}}

\newcommand{\M}{\mathbf{M}}

\newcommand{\sdot}[2]{\overset{\lower0.1em\hbox{$\scriptscriptstyle #2$}}{#1}}

\title{Riemannian Motion Policy Fusion through\\Learnable Lyapunov Function Reshaping}

\author{
Mustafa Mukadam\textsuperscript{1},  Ching-An Cheng\textsuperscript{1}, Dieter Fox\textsuperscript{2,3}, Byron Boots\textsuperscript{2,3}, and Nathan Ratliff\textsuperscript{3}\\
\textsuperscript{1}Georgia Institute of Technology, USA\\
\textsuperscript{2}University of Washington, USA\\
\textsuperscript{3}NVIDIA, USA
}

\begin{document}
\maketitle

\vspace{-8mm}
\begin{abstract}
\flow is a recently proposed policy-fusion framework based on differential geometry. While \flow has demonstrated promising performance, it requires the user to provide sensible subtask policies as Riemannian motion policies (RMPs: a motion policy and an importance matrix function), which can be a difficult design problem in its own right. We propose \newflow, a variation of \flow, to address this issue. \newflow supplements \flow with weight functions that can hierarchically reshape the Lyapunov functions of the subtask RMPs according to the current configuration of the robot and environment. This extra flexibility can remedy imperfect subtask RMPs provided by the user, improving the combined policy's performance. These weight functions can be learned by back-propagation. Moreover, we prove that, under mild restrictions on the weight functions, \newflow always yields a globally Lyapunov-stable motion policy. This implies that we can treat \newflow as a structured policy class in policy optimization that is guaranteed to generate stable policies, even during the immature phase of learning. We demonstrate these properties of \newflow in imitation learning experiments both in simulation and on a real-world robot.
\end{abstract}

\vspace{-4mm}
\keywords{Reactive motion generation, Structured end-to-end learning}


\vspace{-2mm}
\section{Introduction}\label{sec:intro}
\vspace{-3mm}

Motion planning and control are core techniques to robotics~\citep{urmson2008autonomous,johnson2015team,correll2018analysis}.
Ideally a good algorithm must be both computationally efficient and capable of navigating a robot safely and stably across a wide range of environments. 
Several systems were recently proposed to address this challenge~\citep{2017_rss_system,Mukadam-ICRA-17,cheng2018rmpflow} through closely integrating planning and control techniques. 
In particular, \flow~\citep{cheng2018rmpflow} is designed to combine reactive policies~\citep{khatib1987unified,Nakanishi_IJRR_2008,Peters_AR_2008,IjspeertDMPs2013,lo2016virtual} and planning~\citep{RIEMORatliff2015ICRA}.
%
Based on differential geometry, \flow offers a unified treatment of the nonlinear geometries arising from a robot's internal kinematics and task spaces (e.g. environments with obstacles). 
Given user-provided subtask motion policies expressed in the form of Riemannian Motion Policies (RMPs)~\citep{ratliff2018riemannian} (i.e. a second-order motion policy along with a matrix function that acts as a directional importance weight), \flow can synthesize a global motion policy for the full task in an efficient and geometrically consistent manner
%
and has desirable properties
 such as stability and being coordinate-free~\citep{cheng2018rmpflow}.

\flow has been successfully applied in many applications 
\cite{meng2019NeuralAutoNavigation,paxton2019RLDS,sutanto2019TactileServoing,li2019MultiAgentRMPs,li2019stable}. But
\flow is not perfect. Despite its advancement, practical usage difficulties remain. For instance, the user must provide RMPs with matrix functions that properly describe the characteristics of subtask motion policies in order to build an effective \flow system.
Otherwise, the final global policy may have unsatisfactory performance, though still being geometrically consistent (with respect to some meaningless geometric structure).
This poses a challenge for practitioners who are inexperienced in control systems, or for designing policies of tasks where the full state is hard to describe. 

In this paper we introduce a hierarchical Lyapunov function reshaping scheme into \flow to remedy the requirement of providing high-quality subtasks RMPs from the user. 
The modified algorithm, called \newflow, adds a set of multiplicative weight \emph{functions} in the policy fusion step of \flow, which can be manually parametrized or modeled by function approximators (like neural networks). 
In a high level, these weight functions let \newflow adapt between multiple versions of \flow according to the robot's configuration and the environment. (\flow is \newflow with constant weights.)
Therefore, an immediate benefit of our new algorithm is the extra design flexibility added to \flow. 
Compared with \flow, \newflow allows the user to start with simpler subtask RMPs and gradually build up more complex behaviors through the use of weight functions.

Interestingly, these weight functions in general do not just linearly combine outputs of motion policies as in~\citep{slotine1991general,arkin2008governing}. Instead they {hierarchically} reshape the inherent Lyapunov functions 
of the provided subtask policies, overall giving a nonlinear effect on the global policy \newflow creates. %
We prove that \newflow produces a policy that is Lyapunov-stable with respect to this reshaped Lyapunov function given by the weight functions. Therefore, the overall the system is stable, as long as the weight functions are non-negative and non-degenerate.

These properties suggest that we can treat \newflow as a structured policy class in reinforcement/imitation learning and optimize the weight functions to improve the combined policy's performance.
%
%
Importantly, as \newflow remains stable under minor restriction on weight functions, we arrive at a policy class that is guaranteed to be stable, even during the immature phase of learning. 
Thus, \newflow is suitable for learning with safety constraints; 
for example, we can ensure that certain safe policies (like collision avoidance) are the only ones activated when the robot is facing extreme conditions.
These theoretical properties of \newflow are verified in imitation learning tasks, in both simulations and on a real-world robot (Figure~\ref{fig:franka}). Not only did \newflow learn to mimic the expert policy, but it also yielded stable policies throughout the learning.

\begin{figure}[!t]
	\centering
	\begin{subfigure}[b]{0.45\linewidth}
		\centering
		\includegraphics[trim={0 0 0 0},clip,width=0.65\linewidth]{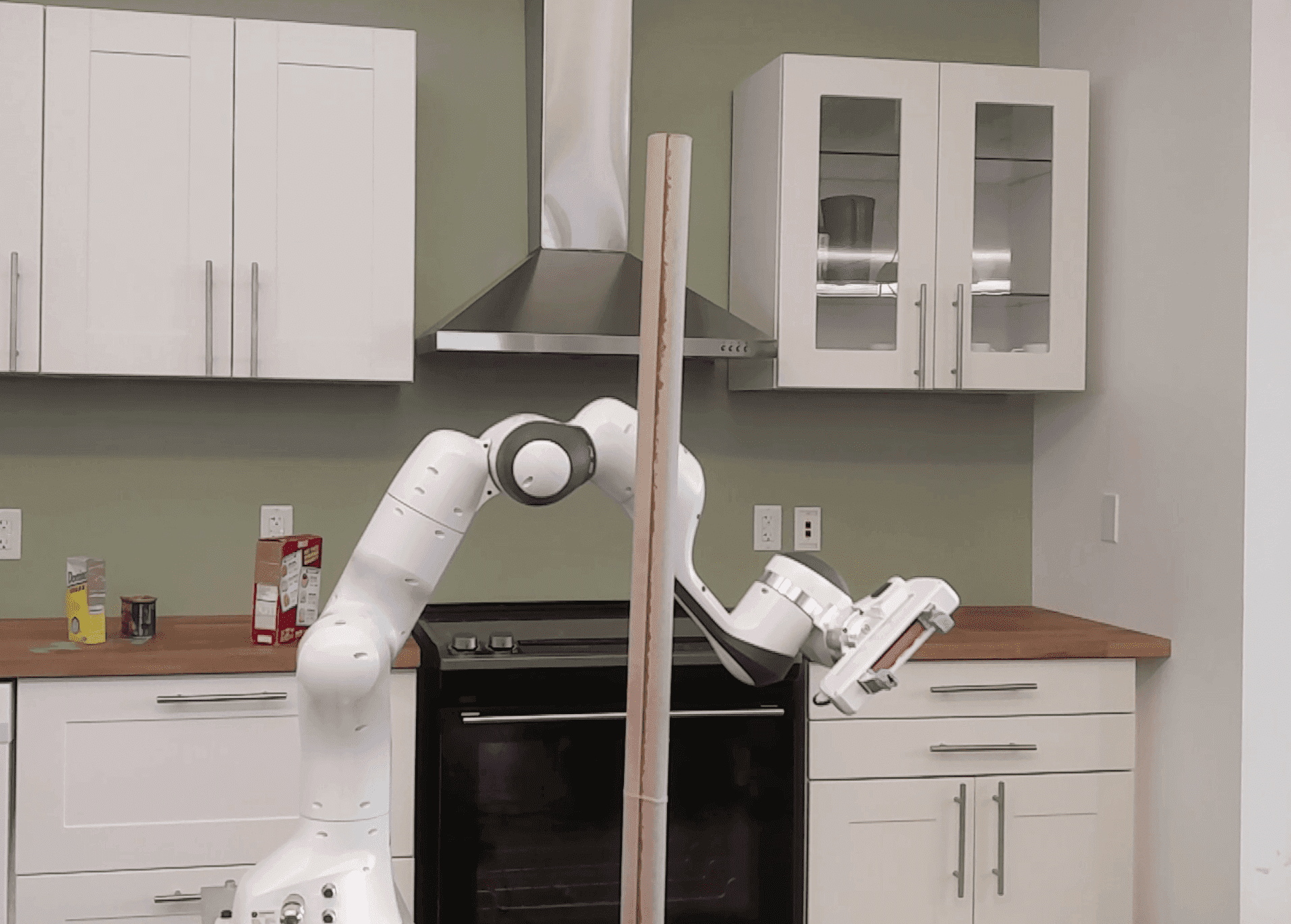}
		\caption{\small A snapshot of the experiment.}
		\label{fig:franka_real}
	\end{subfigure}
	\begin{subfigure}[b]{0.45\linewidth}
		\centering
		\includegraphics[trim={0 0 0 0},clip,width=0.8\linewidth]{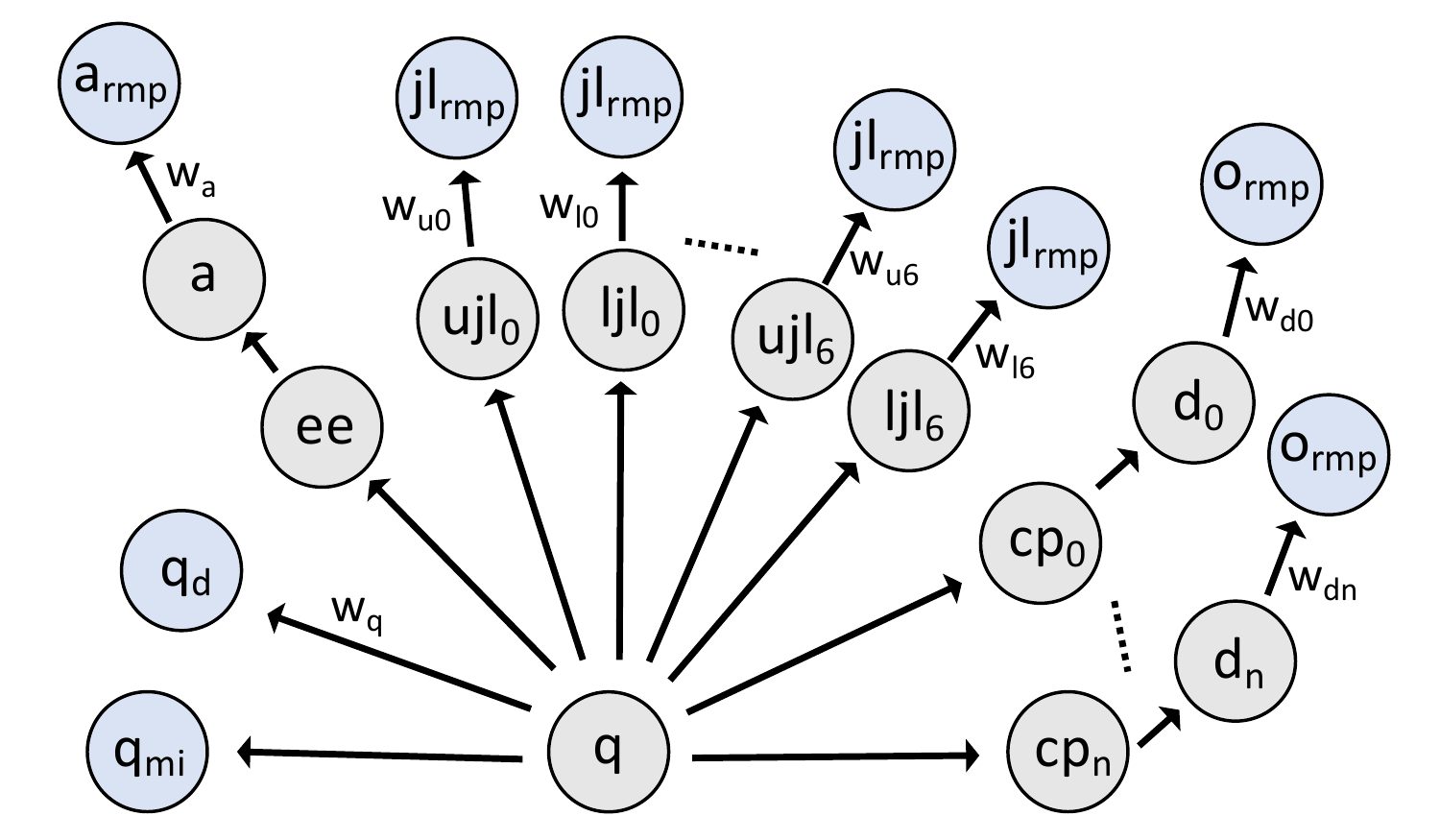}
		\caption{\small The \newtree used for the Franka robot.}
		\label{fig:franka_tree}
	\end{subfigure}
	\vspace{-2mm}
	\caption{\small Franka robot navigating around an obstacle using \newflow with the \newtree.
	Gray nodes show task spaces, blue nodes show subtask RMPs, and weight functions are shown on the respective edges where they are defined.
	See Section \ref{sec:franka exps} for details.}
	\label{fig:franka}
	\vspace{-6mm}
\end{figure}

\textbf{Notation\quad}
We use boldface to denote vectors and matrices. For derivatives, we use both the symbols, $\nabla$ and $\partial$, which are transpose to each other: for $\x \in \R^m$ and a differential map $\y:\R^m \to \R^n$, we write $\nabla_\x \y(\x) = \partial_\x \y(\x)^\t \in \R^{m \times n}$.  For convenience we use  $[\cdot]_{\cdot}^\cdot$ to denote horizontal concatenation in composing a matrix; for example, we write $\M = [\mbb_i]_{i=1}^m \in \R^{m\times m}$ for $\mbb_i \in \R^m$. We use $\R^{m\times m}_{+}$ to denote symmetric, positive semi-definite matrices in $\R^{m\times m}$. 
We  assume all manifolds and maps are sufficiently smooth. Coordinates of manifolds mentioned here will be assumed local.

\vspace{-4mm}
\section{Background}
\vspace{-2mm}

\subsection{Motion Policies}
\vspace{-2mm}

We treat a robot's configuration as a point on some smooth manifold and model its motion through differential equations. 
Assume the robot has been feedback linearized. We are interested in motions that can be described by second-order differential equations. We call these differential equations, \emph{motion policies}, as they essentially define how the robot reacts given the current state (i.e. the configuration and the velocity).  Suppose the robot lives on a $d$-dimensional manifold $\CC$ (the configuration space) with coordinate $\q \in \R^d$. We say a map $\pi$ is a {motion policy} on $\CC$ if the robot travels according to the differential equation $\qdd = \pi(\q, \qd)$, where $\hspace{0.2ex}\dot{}\hspace{0.2ex} $ denotes time derivative. 
%


While motion policies can be specified directly on the configuration space $\CC$, it is often more natural to define them indirectly on the \emph{task space} $\TT$ (another manifold) that describes the target application and then transform them back to the configuration space $\CC$. 
This is the central idea underlined in operational space control~\citep{khatib1987unified}.
For instance, suppose $\TT$ has a coordinate $\x$ that is related to $\CC$ through a task map $\psi$ (i.e. $\x = \psi(\q)$). 
A popular way to design motion policies is through the analogy of a mass-spring-damper system~\cite{khatib1987unified,Nakanishi_IJRR_2008,Peters_AR_2008,IjspeertDMPs2013}. These policies can be written in the task space $\TT$ as 
\begin{align} \label{eq:mass-spring-damper system}
\hspace{-2mm}\Mb(\x)\xdd + \C(\x,\xd)\xd = -\Kb(\x-\x_g) -\B\xd 
\end{align}
where 
$\Mb(\x)\succ 0 $ is the inertia matrix (this inertia might not necessarily be the physical inertia of the robot), $\C(\x,\xd)\xd$ is the Coriolis term with respect to $\Mb$, $\Kb\succeq0$ is the stiffness matrix, $\B\succ0$ is the damping matrix, and $\x_g$ is the goal in $\TT$. 
Using~\eqref{eq:mass-spring-damper system}, the robot's behavior can be easily understood: it travels toward $\x_g$ along a trajectory regulated by  damper $\B$. 


\vspace{-2mm}
\subsection{\flow}\label{sec:tree}
\vspace{-2mm}

However, specifying a global task-space policy, like above, can sometimes still be a daunting task, as the task requirement becomes complicated.
\flow is designed to address this issue~\cite{cheng2018rmpflow}. Rather than asking the user to provide a \emph{global} task-space policy, \flow asks for only motion policies for \emph{subtasks} of the original problem. This potentially can be much simpler.
For instance, directly designing a reaching policy for a cluttered environment is non-trivial, but individually specifying policies for obstacle-free goal reaching and collision avoidance is more straightforward.

Inspired by geometric control theory~\citep{bullo2004geometric}, \flow provides a rigorous framework for policy fusion 
with theoretical guarantees, such as stability and geometric consistency.
In implementation, \flow is realized by a data structure called \emph{\tree}, and a set operations called \emph{\algebra}.
Below we highlight major features of each component.

\textbf{\tree\quad}
An \tree (e.g. \cref{fig:franka_tree} but without the weights $w_{\cdot}$) is a directed tree, which expresses the task map $\psi$ as a sequence of basic maps. 
The \tree serves two major purposes: (i) it provides a language for the user to specify the connections between different subtasks, and (ii) it allows \flow to reuse those basic computations inside $\psi$ to achieve efficient policy fusion.
In the \tree, each node represents an RMP and its state; and each edge represents a transformation between manifolds in the user given decomposition of $\psi$. 
Particularly, the leaf nodes consist of the user-defined subtask RMPs, and the root node maintains the RMP of the global policy $\pi$ on $\CC$.

\tree uses RMP~\citep{ratliff2018riemannian} to describe motion policies on manifolds.
Consider an $m$-dimensional manifold $\MM$ with coordinate $\x \in \R^m$ and a motion policy $\ab$ on $\MM$ (i.e. $\xdd = \ab(\x,\xd)$). An RMP pairs the motion policy $\ab$ with an \textit{abstract} inertia matrix $\Mb(\x,\xd) \in \R^{m\times m}_+$, a function of \emph{both} $\x$ and $\xd$ that describes the directional importance of $\ab$ at the current state $(\x,\xd)$ (see~\cite{cheng2018rmpflow} for details). 
The RMP of $\ab$ can be written in the canonical form $(\ab, \M)^\MM$ or in the natural form $[\f, \M]^\MM$, in which $\f = \M \ab$ is called the force map. Note that $\f$ and $\Mb$ are not necessarily physical quantities, and that the motion policy in an RMP is not necessarily in the form of~\eqref{eq:mass-spring-damper system}.



\textbf{\algebra\quad}
\flow uses the \algebra to combine the subtask policies at leaf nodes into a global policy on the configuration space at the root node. \algebra consists of three operators:

(i) \pushforward propagates the state $(\x,\xd)$ of a node in the \tree to update the states of its $K$ child nodes. The state of its $i$th child node is computed as $(\y_i, \yd_i) = (\psi_{i}(\x) , \J_i (\x) \xd )$, where $\psi_{i}$ is the transformation $\y_i = \psi_{i}(\x)$ and $\J_i = \partial_\x \psi_{i}$ is the Jacobian matrix.

(ii) \pullback propagates the RMPs from the $K$ child nodes to the parent node as  $[\f, \M ]^\MM$ with
\begin{align} \label{eq:pullback}
	\f =\textstyle \sum_{i=1}^{K} \J_i^\t (\f_{i} - \M_i \dot{\J}_i \xd) 
	\qquad\text{and}\qquad
	\M =\textstyle \sum_{i=1}^{K} \J_i^\t \M_i \J_i
\end{align}
where  $[\f_{i}, \M_i]^{\NN_i}$  is the RMP of the $i$th child node in the natural form.

(iii) \resolve maps an RMP from its natural form $[\f, \M]^{\MM}$ to its canonical form $(\ab, \M)^{\MM}$ with $\ab = \M^{\dagger} \f$, where $\dagger$ denotes Moore-Penrose inverse. 

To compute the global policy $\pi$ at time $t$, \flow first performs a forward pass by recursively calling \pushforward. 
Then it performs a backward pass by recursively calling \pullback and computes $[\fb_{\texttt{r}}, \M_{\texttt{r}}]^\CC$ at the root. 
Finally, the global policy $\pi=\ab_{\texttt{r}}$ is generated by using \resolve. 
%
%
Loosely speaking, the global policy $\pi$ can be viewed as a weighted combination of the subtask policies. This can be seen by rewriting~\eqref{eq:pullback} as $\ab = \M\f = (\sum_{i=1}^{K} \J_i^\t \M_i \J_i)^{-1} \J_i^\t (\M_i\ab_{i} - \M_i \dot{\J}_i \xd)$ (which is linear combination of child policies $\ab_i$ plus some curvature correction due to $\dot{\Jb}_i$).

\vspace{-1mm}
\subsection{Theoretical Properties of \flow and GDSs} \label{sec:GDS}
\vspace{-1mm}

\flow is proved to be Lyapunov stable and coordinate-free, when the subtask policies belong to Geometric Dynamical Systems (GDSs)~\cite{cheng2018rmpflow}. 
GDSs are a family of dynamical systems on manifolds that generalizes~\eqref{eq:mass-spring-damper system} to have \emph{velocity-dependent} inertias.
Let $\MM$ be an $m$-dimensional manifold with coordinate $\x \in \R^m$. Let $\Gb: \R^m \times \R^m \to \R^{m\times m}_{+}$  be a \emph{metric} matrix, $\B: \R^m \times \R^m \to \R^{m\times m}_{+}$ be a \emph{damping} matrix, and $\Phi: \R^m \to \R$ be a \emph{potential} function, which is lower bounded.
A dynamical system on $\MM$ is said to be a \emph{GDS} $(\MM, \Gb, \B, \Phi)$ if it follows
\begin{align} \label{eq:GDS}
\Mb(\x,\xd)  \xdd 
+ \bm\xi_{\G}(\x,\xd)  = - \nabla_\x \Phi(\x) - \Bb(\x,\xd)\xd,
\end{align}
in which  
$\Mb(\x,\xd)  \coloneqq\Gb(\x,\xd) + \bm\Xi_{\G}(\x,\xd)$, $\bm\Xi_{\G}(\x,\xd) \coloneqq \textstyle \frac{1}{2} \sum_{i=1}^m  \dot{x}_i \partial_{\xd} \gb_{i}(\x,\xd)$, 
$\bm\xi_{\G}(\x,\xd) \coloneqq\textstyle  \sdot{\Gb}{\x}(\x,\xd) \xd - \frac{1}{2} \nabla_\x (\xd^\t \Gb(\x,\xd) \xd)$, 
and $\sdot{\Gb}{\xb}(\x,\xd) \coloneqq  [\partial_{\x}  \gb_{i} (\x,\xd) \xd]_{i=1}^m$.
The term $\Mb$ is again called the inertia matrix, despite being a function of both $\x$ and $\xd$.
The \emph{curvature} terms $\bm\Xi(\x,\xd)$ and $\bm\xi(\x,\xd)$ are generated from the dependency of $\Gb(\x,\xd)$ on $\x$ and $\xd$;
if $\G(\x,\xd) = \G(\x)$, then $\G(\x) = \M(\x)$ and  $\bm\xi_{\Gb}(\x,\xd) = \Cb(\x,\xd) \xd$ in~\eqref{eq:mass-spring-damper system}. 
In view of this, a GDS extends \eqref{eq:mass-spring-damper system} to have general potentials and velocity-dependent metrics, which is useful in modeling collision avoidance behaviors~\cite{cheng2018rmpflow}.


The behavior of a GDS $(\MM, \Gb, \B, \Phi)$ is characterized by the Lyapunov function
\begin{align} \label{eq:energy}
\textstyle
V(\x, \xd) = \frac{1}{2} \xd^\t \G(\x,\xd) \xd + \Phi(\x). 
\end{align}
It can be shown that the stability property of \flow is governed by a Lyapunov function in a similar form~\cite{cheng2018rmpflow}, when the leaf-node policies are GDSs. An RMP $(\ab, \Mb)^\MM$ is a GDS if its motion policy is $\ab = \M(\x,\xd)^{-1}(- \nabla_\x \Phi(\x) - \Bb(\x,\xd)\xd-\bm\xi_{\G}(\x,\xd))$.
\begin{theorem} \label{th:flow property}
\textnormal{\cite{cheng2018rmpflow} }
Suppose an \tree has $K$ leaf nodes of GDSs $(\TT_k, \Gb_k, \B_k, \Phi_k)$ with Lyapunov function $V_k$ in~\eqref{eq:energy}, for $k=1,\dots,K$.
Let  
$
V_{\texttt{r}} = \sum_{k=1}^{K} V_k
$
be a Lyapunov candidate. 
\begin{enumerate}\vspace{-3mm}
\item 	If $\M_{\texttt{r}}$ of the root-node RMP on $\CC$ is positive definite, then
$ \dot{V}_r = - \sum_{k=1}^{K} \xd_k^\t \B_k \xd_k \leq 0 $.
\vspace{-2mm}
\item If further $\sum_{k=1}^{K} \J_k^\t \Gb_k \J_k \succ 0 $ and  $\sum_{k=1}^{K} \J_k^\t \B_k \J_k \succ 0$, the system converges forwardly to $\{(\q,\qd) : \nabla_\q \Phi_{\texttt{r}}(\q) = 0 , \qd = 0 \}$, where $\J_k = \partial_\q \x_k$ and $\Phi_{\texttt{r}}(\q) = \sum_{k=1}^{K} \Phi_k(\x_k(\q))$. 
\end{enumerate}\vspace{-1mm}
\end{theorem}

\vspace{-3mm}
\section{\newflow} \label{sec:new algorithm}
\vspace{-2mm}
%
%

\flow provides a control-theoretic framework for combining subtask policies. However, certain limitations exist. Particularly, it requires the user to provide sensible inertia matrices (cf. \cref{sec:tree}) to describe the subtask policies' characteristics in the leaf-nodes RMPs; failing to do so may result in a global policy with undesirable performance, albeit still being geometrically consistent with the meaningless geometric structure induced by the bad inertia matrices.
%

In this work, we propose a modified algorithm, \newflow, which 
adds extra flexibilities into \flow to address this difficulty. 
The main idea is to introduce an additional set of weight functions as gates to switch on and off the child-node policies in the \tree, 
based on the current state of the robot and the environment. 
These functions can either be designed by hand, or be parameterized as function approximators (like neural networks) which are then learned end-to-end from data (see~\cref{sec:learning}). 
As a result, \newflow can combine simpler/imperfect subtask RMPs into a better global policy, lessening the burden on the user to directly provide high-quality subtasks RMPs.

\newflow modifies \tree and \algebra into \newtree and \newalgebra, respectively.
\newtree 
augments each node in \tree with extra information and each edge with a weight function; \newalgebra 
replaces \pullback with \newpullback. 
Below we define these modifications. In addition, we show that \newflow retains the nice structural properties of \flow: under mild conditions on the weight functions, the global policy of \newflow is Lyapunov stable.
Later in Section~\ref{sec:learning}, we will show how to learn the weight functions in \newflow from data. 

\vspace{-1mm}
\subsection{\newtree and \newalgebra} \label{sec:modifcation}
\vspace{-1mm}
\textbf{Modified node\quad}
In addition to the RMP and its state, each node in \newtree also stores the \emph{values} of a scalar function $L$ 
and the metric matrix $\G$.  
When a leaf-node RMP is a GDS, $\G$ is defined as~\eqref{eq:GDS} and $L = \frac{1}{2} \xd^\t \G \xd - \Phi(\x)$ (analogue of the Lagrangian in mechanical systems).

\textbf{Modified edge\quad}
Each edge in an \newtree has in addition a weight function. 
This weight is a function of the parent-node configuration and some auxiliary state (which describes the task at hand, e.g., the location of the goal in a reaching task). 
 
\textbf{Modified pullback\quad}
We modify \pullback into \newpullback to use the weight functions on edges to combine child-node RMPs. 
For the parent and child nodes given in~\eqref{eq:pullback}, we set instead 
\begin{small}
\begin{align} \label{eq:modifed pullback}
\f = \sum_{i=1}^{K} w_i \J_i^\t (\f_{i} - \M_i \dot{\J}_i \xd) +  \hb_i,
\quad
\M = \sum_{i=1}^{K} w_i \J_i^\t \M_i \J_i, 
\quad
\G = \sum_{i=1}^{K} w_i \J_i^\t \G_i \J_i,
\quad 
L = \sum_{i=1}^{K} w_i  L_i
\end{align}
\end{small}%
where $\hb_i = L_i \nabla_\x w_i-  (\xd^\t \nabla_\x w_i) \J_i^\t \G_i \J_i \xd $.
From~\eqref{eq:modifed pullback}, we see that \newpullback does \emph{not} simply linearly combine child-node motion policies. It adds a correction term $\hb_i$, which is designed to anticipates the change of weighting $w_i$ so that the system remains stable. When applied recursively in policy generation, it would \emph{hierarchically} reshape the Lyapunov functions (see Section~\ref{sec:advantages}).

\vspace{-3mm}
\subsection{Stability}
\vspace{-2mm}
We show \newflow is also Lyapunov stable like \flow.
To state the stability property, let us introduce additional notation to describe the functions in the \newtree. We will use $(i;j)$ to denote the $i$th node in depth $j$ of an \newtree and we use $C_{(i;j)}$ to denote the indices of its child nodes. For example, node $(1;0)$ denotes the root node (also denoted as $\texttt{r}$). In addition,  we will refer to the functions on the edges using the indices of the child nodes, e.g., 
the Jacobian of the transformation to the $i$th node in depth $j$ is denoted as $\J_{(i;j)}$. 
We show the stability property of \newflow when all the leaf nodes are of GDSs, like Theorem~\ref{th:flow property}. The proof is given in Appendix~\ref{app:proof}.

\begin{theorem} \label{th:newflow property}
	Suppose an \newtree has leaf-node policies as GDSs with Lyapunov functions given as~\eqref{eq:energy}. 
	Define  $V_{(i;j)}$,  $\B_{(i;j)}$, and $\Phi_{(i;j)}$
	on the \newtree through the recursion
	\begin{align} \label{eq:recursive law}
	\begin{split}
	\textstyle
	V_{(i;j)} =  \sum_{k\in C_{(i;j)}} w_{(k;j+1)} V_{(k;j+1)}, 
	\quad&\textstyle
	\B_{(i;j)} = \sum_{k\in C_{(i;j)}} w_{(k;j+1)} \J_{(k;j+1)}^\t \B_{(k;j+1)} \J_{(k;j+1)}\\
	&\textstyle \hspace{-20mm}\Phi_{(i;j)}= \sum_{k\in C_{(i;j)}} w_{(k;j+1)} \Phi_{(k;j+1)}
	\end{split}
	\end{align}
	in which the boundary condition is given by the leaf-node GDSs. 
	Let $V_{\texttt{r}}$ be a Lyapunov candidate.
	\begin{enumerate}\vspace{-3mm}
	\item 	If $\M_{\texttt{r}}\succ0$, then $\dot{V}_{\texttt{r}} = - \qd^\t \B_{\texttt{r}} \qd \leq 0  $. 	
	\vspace{-2mm}
	\item 	If further $\G_{\texttt{r}}, \B_{\texttt{r}}\succ0$, the system converges forwardly to $\{(\q,\qd) : \nabla_\q \Phi_{\texttt{r}}(\q) = 0 , \qd = 0 \}$. 
	\end{enumerate}
\end{theorem}
\cref{th:newflow property} shows that the system is Lyapunov stable with respect to $V_{\texttt{r}}$. To satisfy the conditions required in \cref{th:newflow property}, a sufficient condition is to select leaf-node GDSs with certain monotone metrics~\cite[Theorem 2]{cheng2018rmpflow} and have positive weight functions on edges. Therefore, in addition to the conditions needed by \flow, 
\newflow only imposes mild constraints on the weight functions. This is a useful feature when the weight functions are learned from data, because \cref{th:newflow property} essentially guarantees the output policy is always stable even in the premature stage of learning.

Note that it is straightforward to extend \newtree 
to include, in~\eqref{eq:pullback}, an extra time-varying term that vanishes as $t\to\infty$ (like the one used in DMPs~\cite{IjspeertDMPs2013}) and to consider time-varying potentials (e.g. in tracking applications). We omit discussions about these generalizations due to space limitation.

\vspace{-3mm}
\subsection{Advantages of \newflow over \flow} \label{sec:advantages}
\vspace{-2mm}
\newflow strictly generalizes \flow. When each weight is constant one, \newflow becomes \flow (i.e. \newpullback is the same as \pullback and Theorem~\ref{th:newflow property} reduces to Theorem~\ref{th:flow property}).
More generally, \newflow allows mixing policies through reweighting their Lyapunov functions, while retaining the nice structural properties of \flow, as shown in Theorem~\ref{th:newflow property}.

In comparison, \newflow has a more flexible way to express policies and compose the subtask Lyapunov functions into the Lyapunov candidate $V_{\texttt{r}}$ in~\eqref{eq:new Lyapunov}.  
Whereas Theorem~\ref{th:flow property} uses the simple summation of subtask energies $V_{\texttt{r}} = \sum_{i=1}^{K} V_i$, 
Theorem~\ref{th:newflow property} effectively uses the Lyapunov function
	\begin{align} \label{eq:new Lyapunov}
	\textstyle	V_{\texttt{r}} =  \sum_{k_1\in C_{(1;0)}}  w_{(k_1;1)} 
	\sum_{k_2 \in C_{(k_1;1)}} \dots 
	\sum_{k_{D} \in C_{(k_{D-1};D-1)}} w_{(k_{D};D)} V_{(k_{D};D)}
	\end{align}
for a depth-$D$ \newtree  (cf.~\eqref{eq:recursive law}) and each weight $w_{(i;j)}$ can be a function of the configuration and auxiliary state of the parent of node $(i;j)$. Therefore, from~\eqref{eq:recursive law} and~\eqref{eq:new Lyapunov}, \newflow can be viewed as a form of hierarchical Lyapunov function reshaping scheme along the hierarchy structure induced by the \newtree.
Consequently, the recursive formulation of \newflow allows the user only to provide basic subtask policies and gradually increase
their expressiveness by the weight functions. 
In contrast, using \flow requires directly specifying subtask policies with complicated behaviors.
We include a concrete example to illustrate the benefit of this extra flexibility in \cref{app:flexibility example}.


\vspace{-2mm}
\subsection{Learning \newflow} \label{sec:learning}
\vspace{-2mm}

\begin{figure}[!t]
	\vspace{-6mm}
	\centering
	\begin{subfigure}[b]{0.6\linewidth}
		\centering
		\includegraphics[width=\linewidth,trim={0 5mm 0 0}]{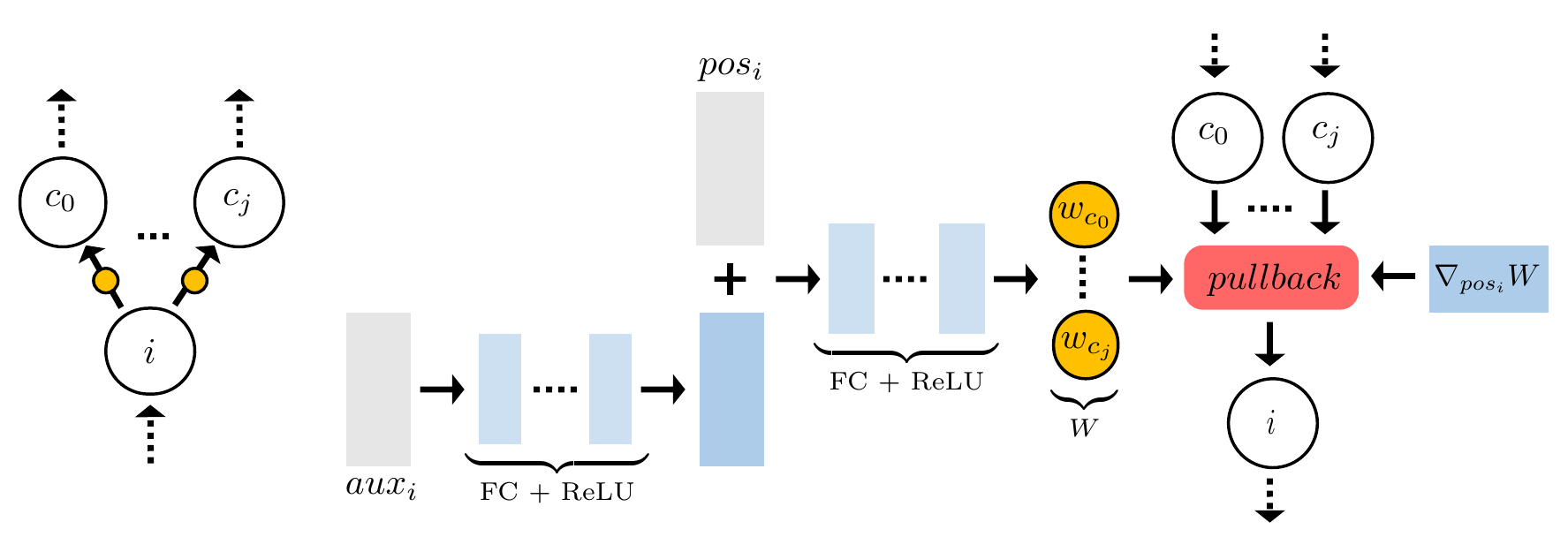}\vspace{-2mm}
		\caption{}
		\label{fig:rmp-flow-net}
		\vspace{-2mm}
	\end{subfigure}
	\begin{subfigure}[b]{0.19\linewidth}
		\centering
		\includegraphics[width=\linewidth,trim={0 6mm 0 0}]{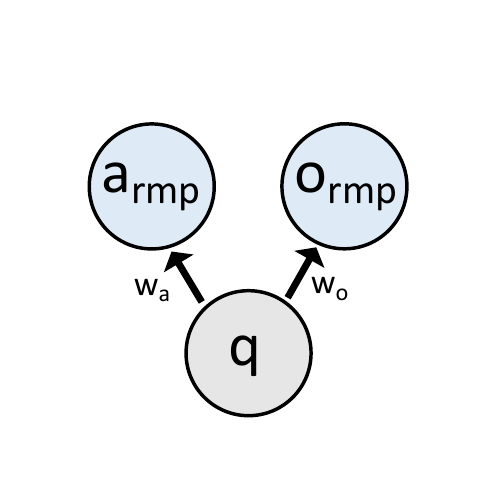}\vspace{-2mm}
		\caption{}
		\label{fig:2d1level}
		\vspace{-2mm}
	\end{subfigure}
	\begin{subfigure}[b]{0.19\linewidth}
		\centering
		\includegraphics[width=\linewidth,trim={0 0 0 0}]{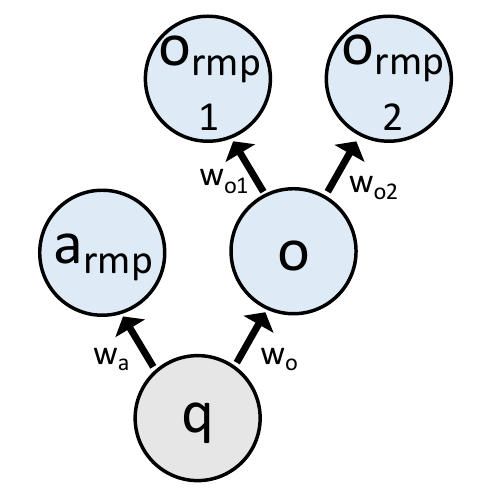}\vspace{-2mm}
		\caption{}
		\label{fig:2d2level}
		\vspace{-2mm}
	\end{subfigure}
	\vspace{-1mm}
	\caption{\small (a) Shows the network used for learning with \newflow, specifically for any node $i$ on the \newtree, with children $c_0, \dots, c_j$. If $i$ is a leaf node, then it is evaluated from the designed RMP policy. The global policy is obtained by applying \resolve on the root node RMP. \newtree used in experiments for (b) \texttt{2d1level} and (c) \texttt{2d2level}.}
	\label{fig:an example weight function}
	\vspace{-6mm}
\end{figure}


We presented a new computational graph, \newflow, which supplements \flow with a set of multiplicative weight functions to achieve extra flexibility in policy fusion. 
In Appendix~\ref{app:learning},  we show these weight functions can be learned by back-propagation, and therefore \newflow can be treated as a parameterized policy class in policy optimization by using computational graph libraries like tensorflow~\cite{abadi2016tensorflow} or pytorch~\cite{paszke2017automatic}.
%
%
Finally, it is important to note that we do not have to learn all the weight functions in a \newtree. If we know that certain leaf-node RMPs have to be turned on, we can adopt a semi-parametric scheme of weight functions. For example, we can design parameterization of the weight functions such that only collision avoidance RMPs are turned on, when the robot is extremely close to an obstacle.  This property is due to the structure of \newtree, which is interpretable, unlike policies purely based on general function approximators. Interpretability allows for prior knowledge (like constraints and preferences) to be easily incorporated into the policy structure.  This feature is particularly valuable for policy learning with safety constraints~\cite{garcia2015comprehensive}. 

\vspace{-3mm}
\section{Experiments}
\vspace{-2mm}

We validate our approach with experiments of imitation learning. The goal is to show that \newflow with an \newtree that is parametrized by randomly initialized neural networks (as in Figure~\ref{fig:an example weight function}) is able to mimic the expert policy's behavior by observing expert demonstrations. This setup simulates the situation where the user of \newflow only provides imperfect subtask policies. We also use these experiments to validate the stability properties of \newflow by studying if the Lyapunov function of the policies generated by \newflow (even the premature ones obtained before learning converges) decay monotonically over time. We perform these experiments with a 2D particle robot and with a Franka Panda 7-DOF robot (video of experiments is available at \href{https://youtu.be/McSrpW-mIq4}{https://youtu.be/McSrpW-mIq4}).

As our aim it not to invent a new imitation learning algorithm, we adopt the most basic approach, behavior cloning~\cite{pomerleau1989alvinn}, in which the demonstrations are purely generated by running the expert policy alone without any active intervention from the learner. The objective of these experiments is to study how well \newflow can recover the behaviors of an expert that is within its effective policy class, and therefore we use a known \newtree with fixed weights as the expert policy. We choose this setting to rule out bias due to mismatches between policy classes, 
because properly handling policy class biases in imitation learning is a non-trivial research question on its own right~\cite{ross2011reduction,ross2014reinforcement,cheng2018convergence,cheng2018fast}. Note that any policy optimization technique can be used to train RMPfusion, including online imitation learning and policy gradient methods, etc.

\vspace{-1mm}
\subsection{2D Robot}
\vspace{-2mm}
We first validate our approach on two problems where a 2D robot is tasked with reaching a goal while avoiding one obstacle (\texttt{2d1level}) or two obstacles (\texttt{2d2level}). The \newtree for these problems are shown in Figure~\ref{fig:2d1level}-\ref{fig:2d2level} and are detailed in Appendix~\ref{app:exp}. The tree structure here is heuristically chosen based on the problem domain, as in RMPflow and typically follows the robot’s kinematic chain and then extends into the workspace and abstract task spaces.
In the \texttt{2d1level} problem, the aim is to show near-perfect recovery of the weights given that the problem is convex in the weight functions. The \texttt{2d2level} problem adds extra complexity to the learning process. It introduces multiplication between weights so the weights cannot be uniquely identified.
The aim here is to show that close-to-expert behavior can still be achieved.

\textbf{Data\quad}
For each problem, the expert policy is generated by the respective \newtree with some fixed assigned weights, which are unknown to the learner. The training data consist of 20 \emph{randomly selected environments} with varying placements and sizes of obstacles. In each environment, the expert is run to generate 50 trajectories from unique initial states, and 60 temporally equidistant data points on each trajectory are recorded. Each data point is a pair of input and output: the input consists of the state (position and velocity) of the 2D particle and the auxiliary state (obstacle location and dimension, goal location) i.e. the meta information about the environment; the output consists of the action (acceleration) as specified by the expert given the input state visited by running the expert policy. Test data are collected by repeating this process with 5 new environments with 10 trajectories in each environment.

\textbf{Unstructured network\quad}
For \texttt{2d2level} we also compare our \newflow \texttt{learner-rmp} with an unstructured neural network \texttt{learner-un}. This is a fully connected feed forward network with similar number of learnable parameters compared to \texttt{learner-rmp}. This network takes robot state and auxiliary state as the inputs, and outputs the acceleration. Our aim with this comparison is to show that an unstructured approach cannot offer any stability or safety guarantees, and with the same amount of data and training underperforms compared to the structured approach.

\textbf{Training\quad}
We use the mean squared error between the action generated by any learner and the action specified by the expert as the loss function for imitation learning.
All learners are trained using RMSprop~\citep{tieleman2012lecture} with a minibatch size of 200 for 20,000 iterations.
The number of iterations were chosen such that learning roughly converged and over-fitting had not happened.

\begin{figure*}[!t]
	\centering
	\begin{subfigure}[b]{0.16\textwidth}
		\centering
		\includegraphics[trim={70 0 70 0},clip,width=1\linewidth]{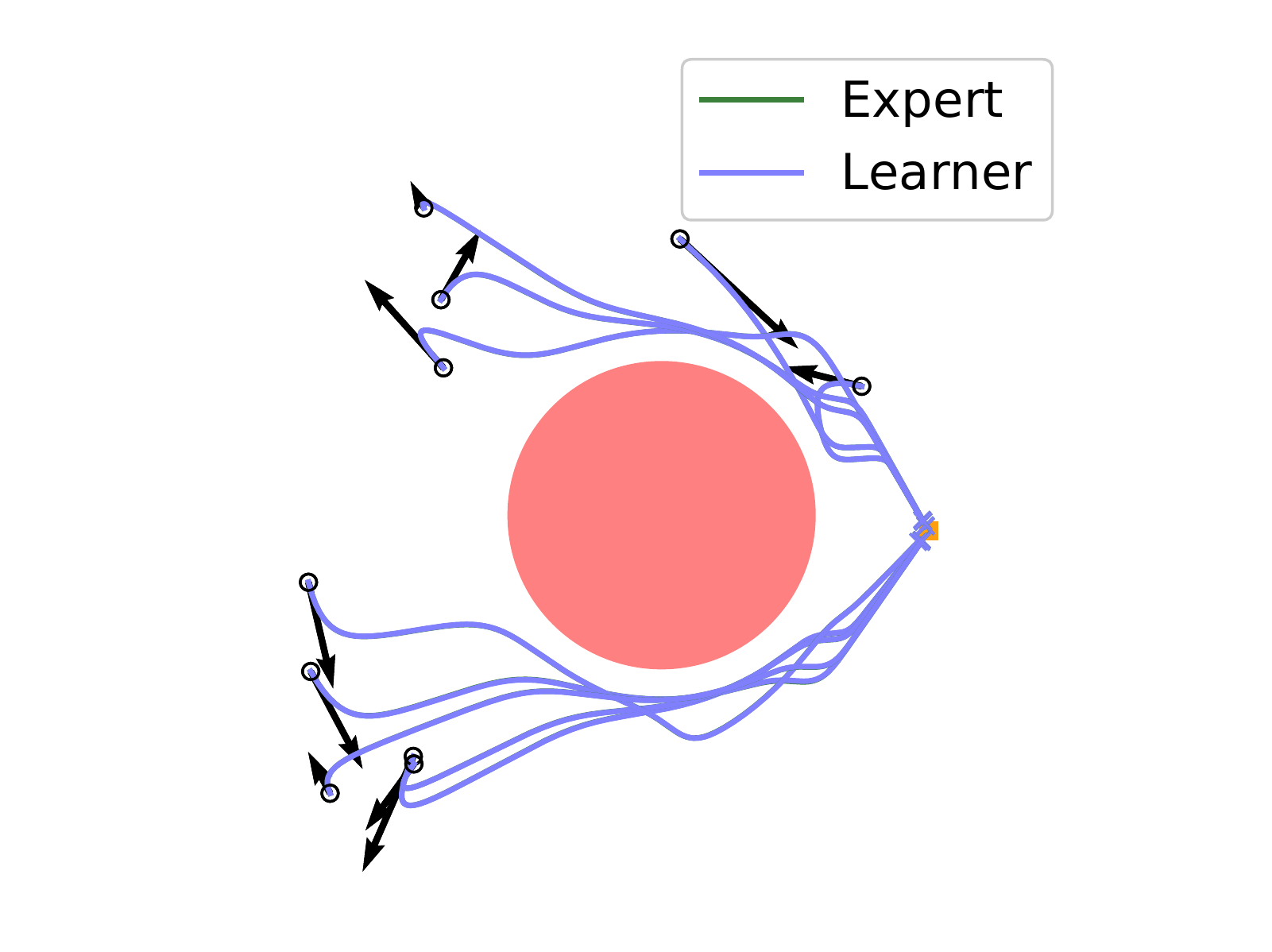}\vspace{-2mm}
		\caption{\texttt{2d1level}}
		\label{fig:2d1l_traj}
	\end{subfigure}
	\begin{subfigure}[b]{0.16\textwidth}
		\centering
		\includegraphics[trim={70 0 70 0},clip,width=1\linewidth]{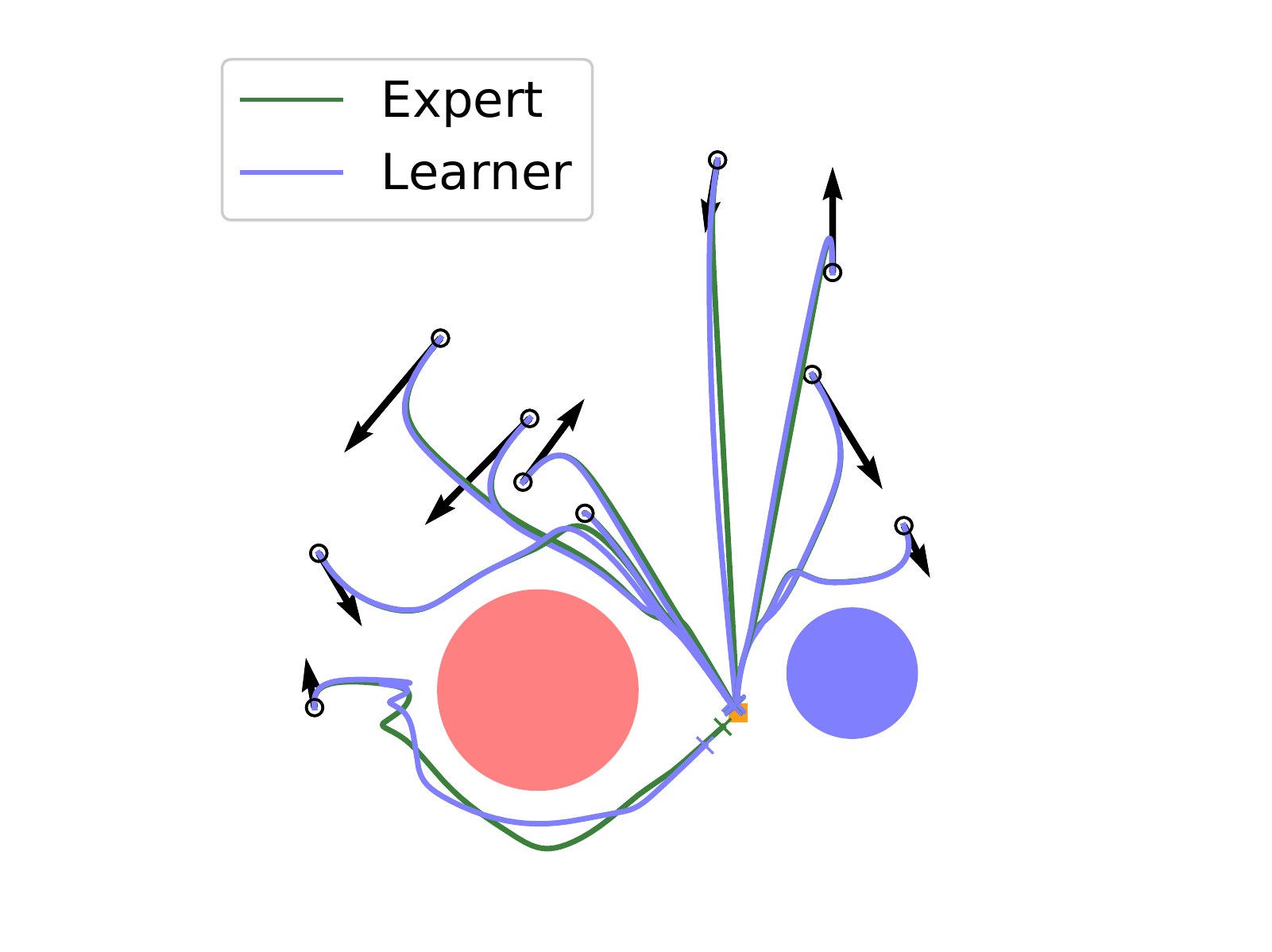}\vspace{-2mm}
		\caption{\texttt{2d2level}}
		\label{fig:2d2l_traj}
	\end{subfigure}
	\begin{subfigure}[b]{0.216\textwidth}
		\centering
		\includegraphics[trim={0 0 0 0},clip,width=1\linewidth]{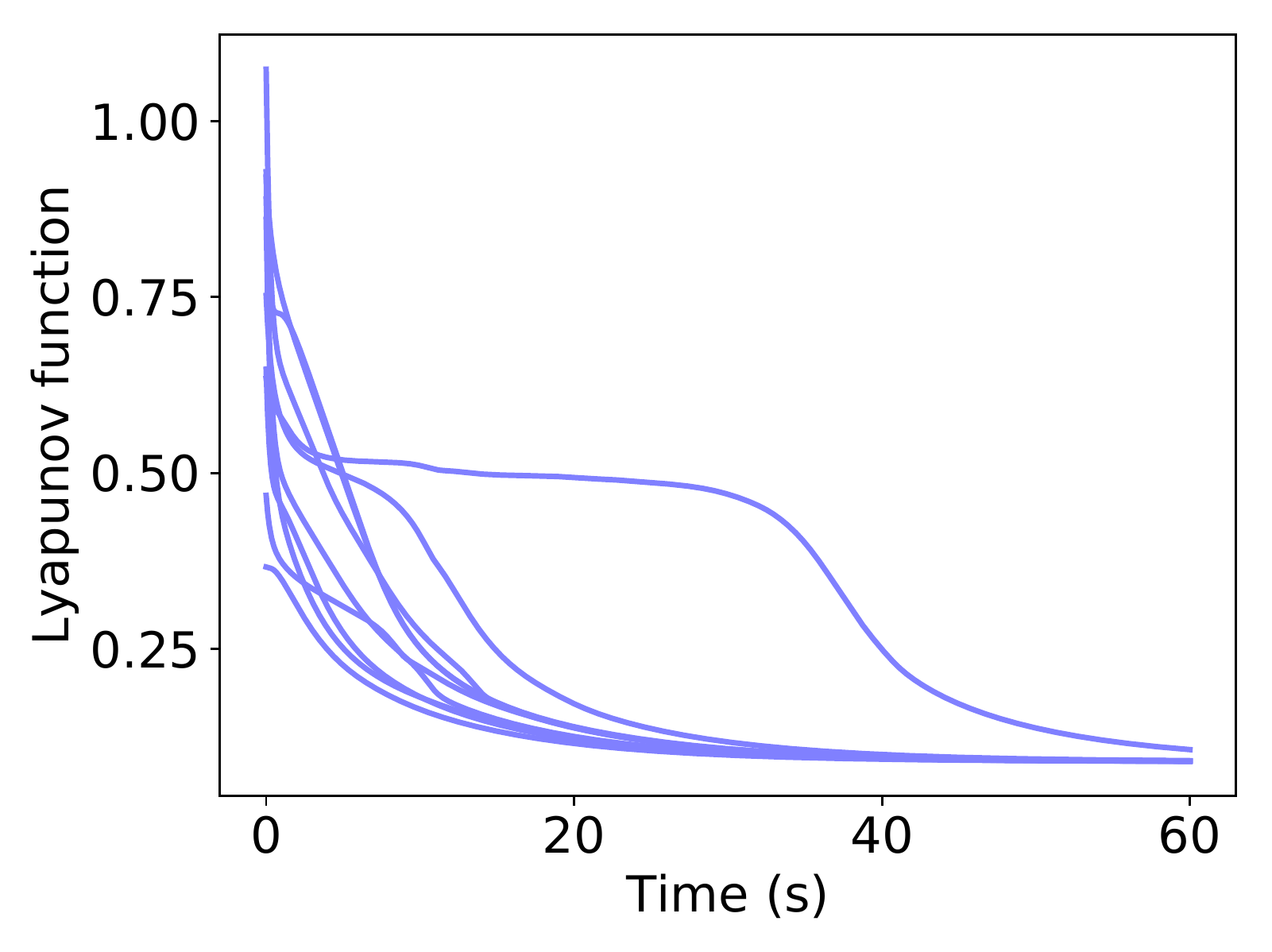}\vspace{-2mm}
		\caption{}
		\label{fig:2d2l_energy}
	\end{subfigure}
	\begin{subfigure}[b]{0.216\textwidth}
		\centering
		\includegraphics[trim={0 0 0 0},clip,width=1\linewidth]{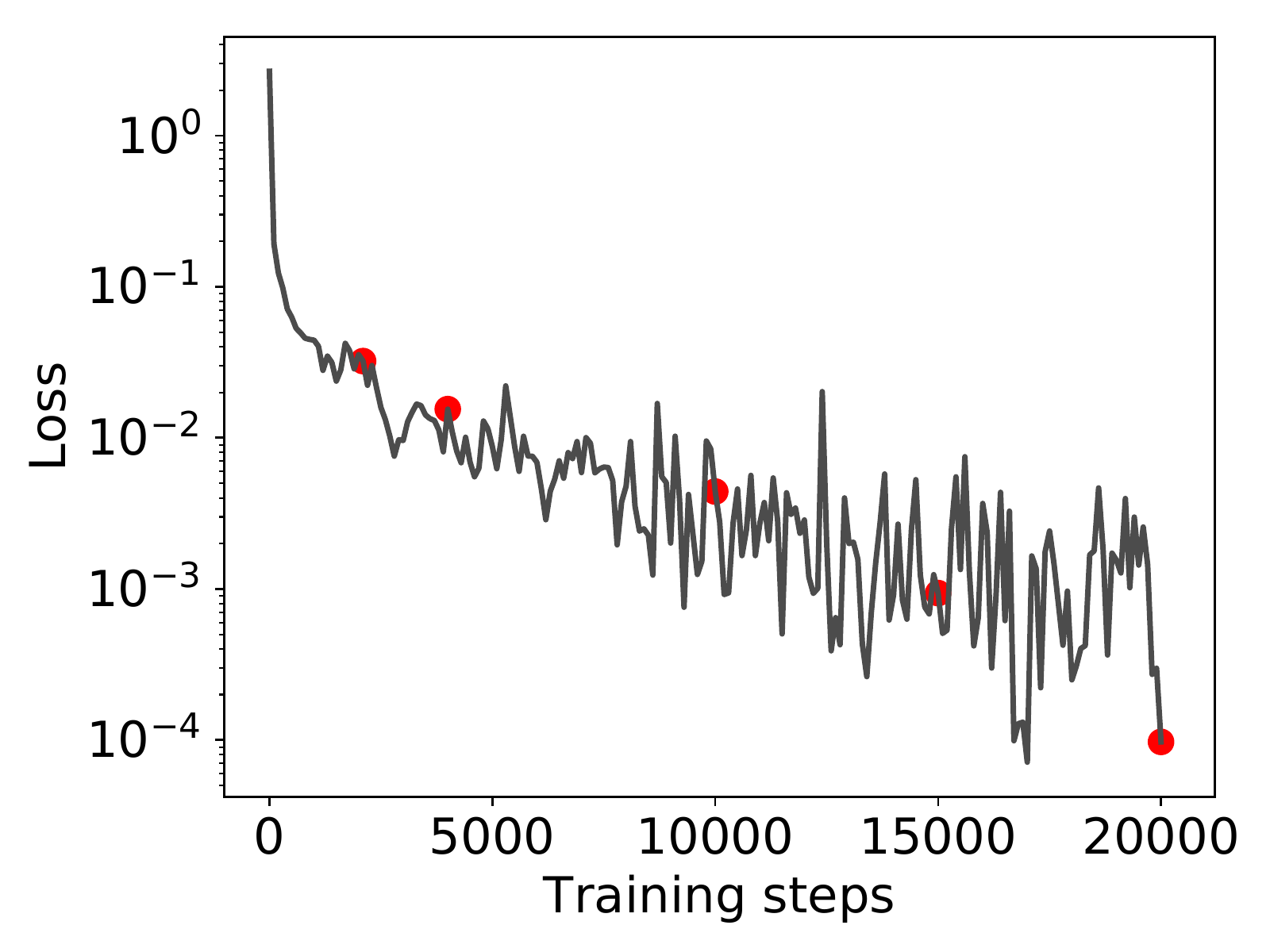}\vspace{-2mm}
		\caption{}
		\label{fig:2d2l_loss}
	\end{subfigure}
	\begin{subfigure}[b]{0.16\textwidth}
		\centering
		\includegraphics[trim={70 0 70 0},clip,width=1\linewidth]{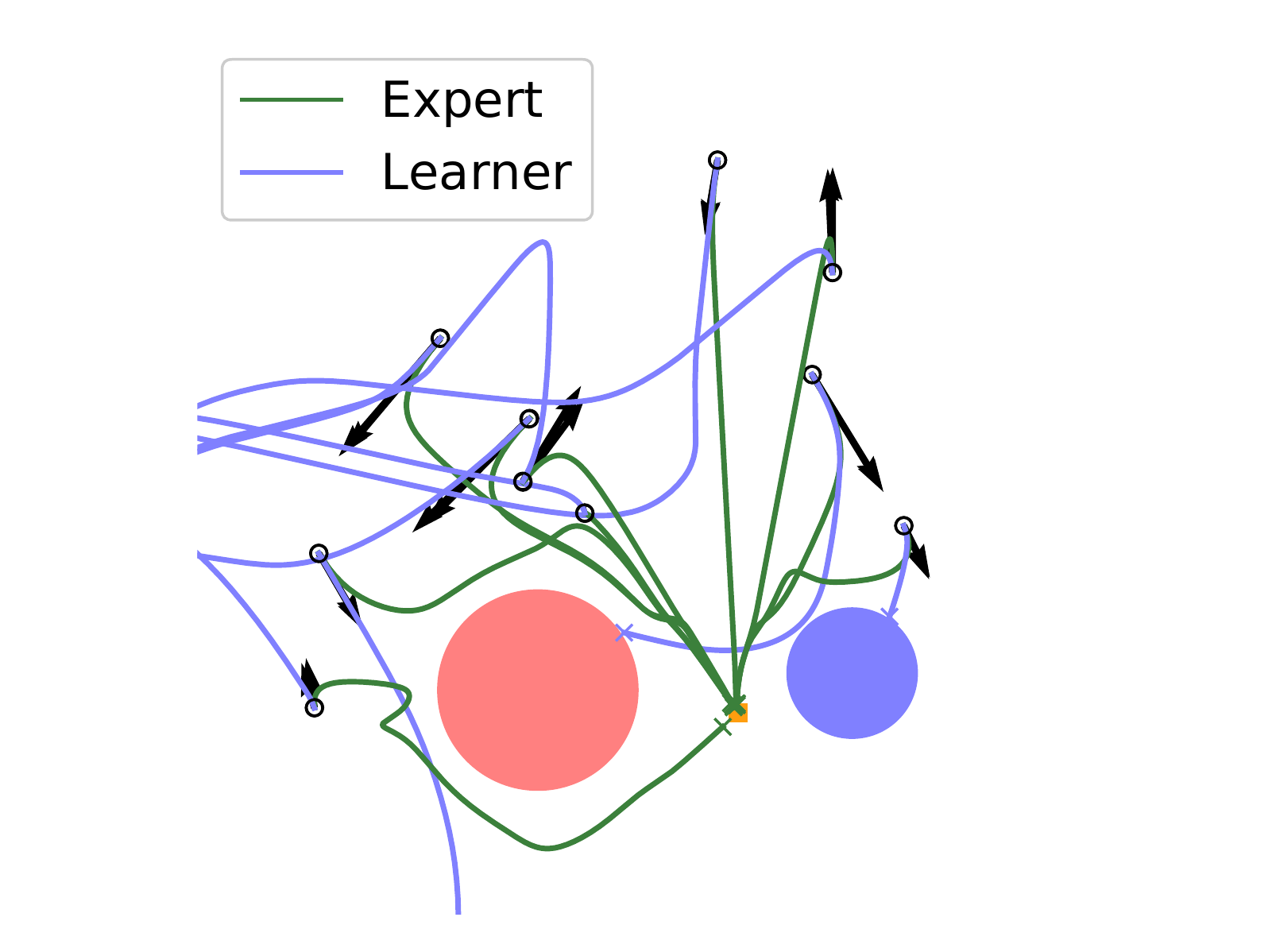}\vspace{-2mm}
		\caption{\texttt{2d2level}}
		\label{fig:2d2l_u_traj}
	\end{subfigure}
	\vspace{-2mm}
	\caption{\small Trajectories generated in by (a)-(b) \texttt{learner-rmp} and (e) \texttt{learner-un}, compared to the expert are shown. Initial state is a black circle for position and black arrow for velocity. The environment has obstacles (red and blue) and goal (orange square). (c) shows the corresponding Lyapunov function for \texttt{learner-rmp} trajectories in (b) while (d) shows its learning curve.}
	\label{fig:2d}
	\vspace{-2mm}
\end{figure*}

\begin{figure}[!t]
	\centering
	\begin{subfigure}[b]{0.19\textwidth}
		\centering
		\includegraphics[trim={100 70 80 100},clip,width=1\linewidth]{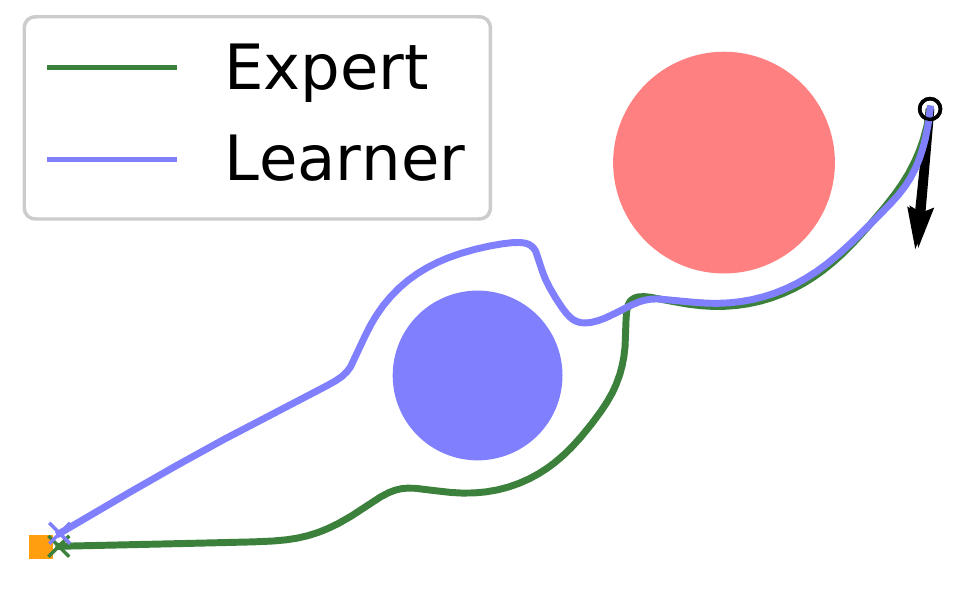}
	\end{subfigure}
	\begin{subfigure}[b]{0.19\textwidth}
		\centering
		\includegraphics[trim={100 70 80 100},clip,width=1\linewidth]{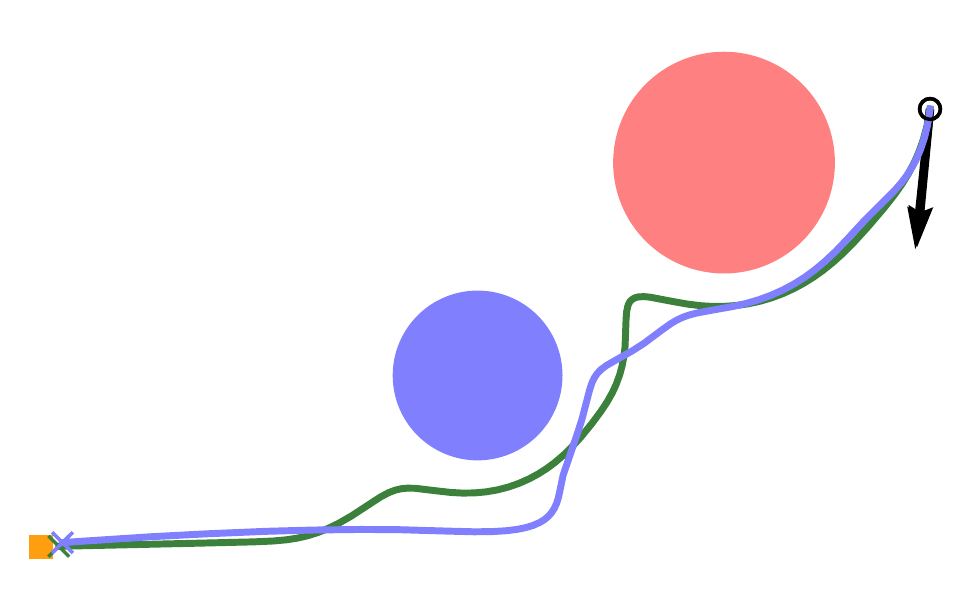}
	\end{subfigure}
	\begin{subfigure}[b]{0.19\textwidth}
		\centering
		\includegraphics[trim={100 70 80 100},clip,width=1\linewidth]{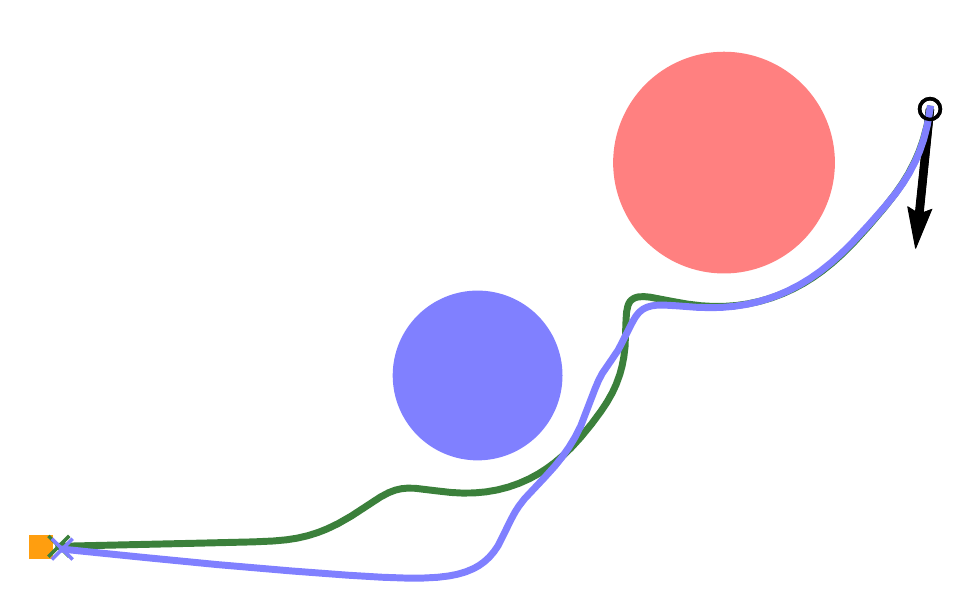}
	\end{subfigure}
	\begin{subfigure}[b]{0.19\textwidth}
		\centering
		\includegraphics[trim={100 70 80 100},clip,width=1\linewidth]{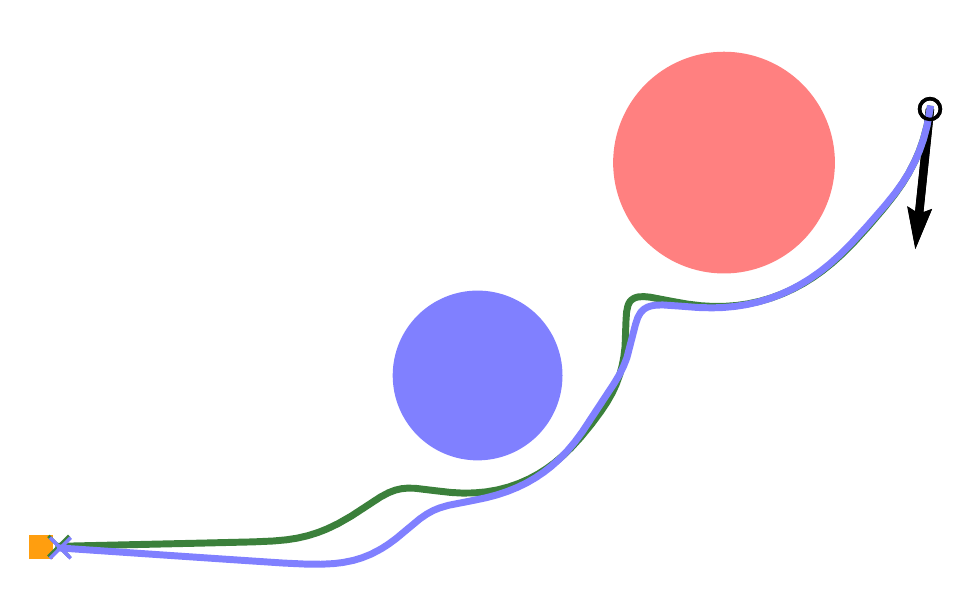}
	\end{subfigure}
	\begin{subfigure}[b]{0.19\textwidth}
		\centering
		\includegraphics[trim={100 70 80 100},clip,width=1\linewidth]{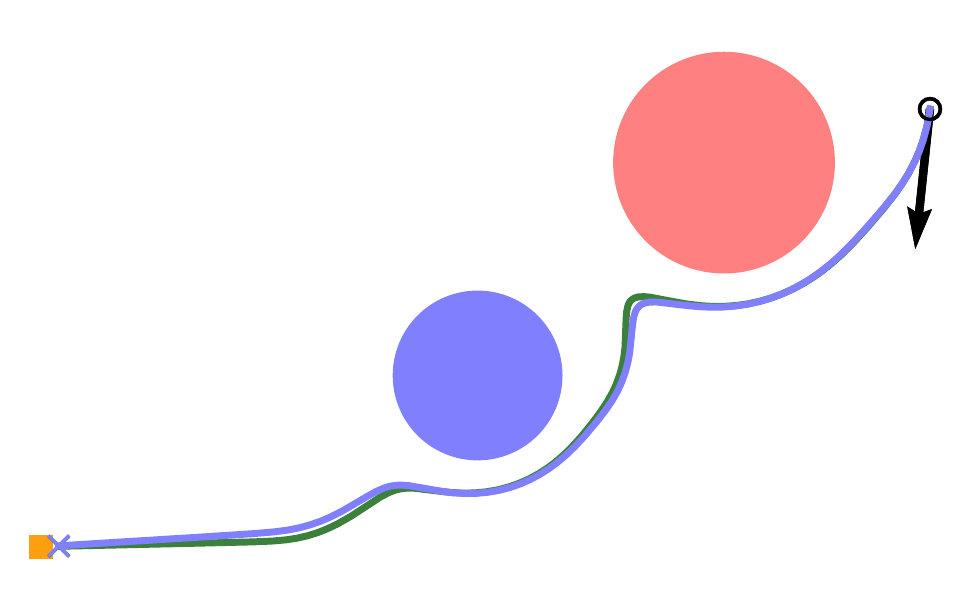}
	\end{subfigure}
	\begin{subfigure}[b]{0.19\textwidth}
		\centering
		\includegraphics[trim={0 0 0 0},clip,width=1\linewidth]{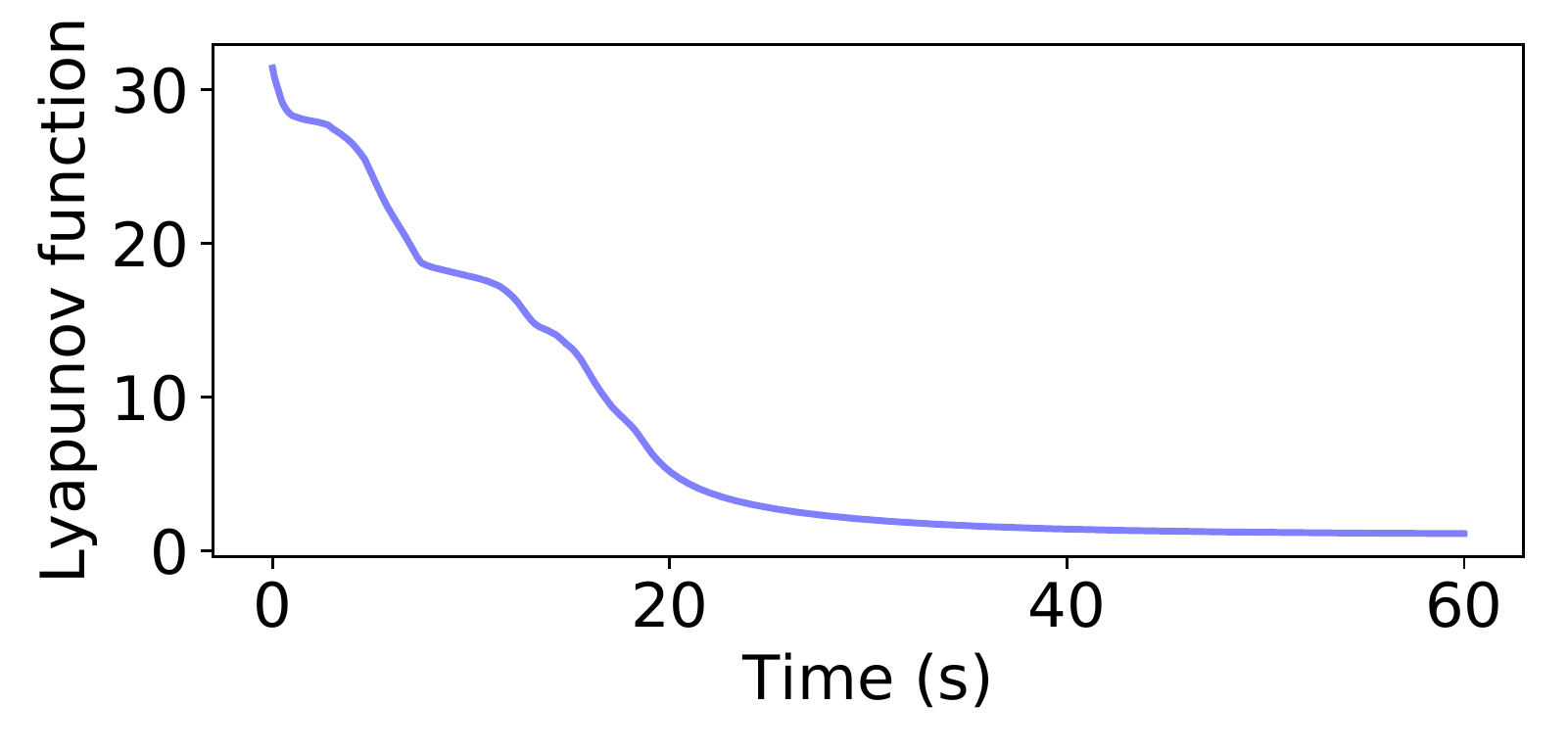}
	\end{subfigure}
	\begin{subfigure}[b]{0.19\textwidth}
		\centering
		\includegraphics[trim={0 0 0 0},clip,width=1\linewidth]{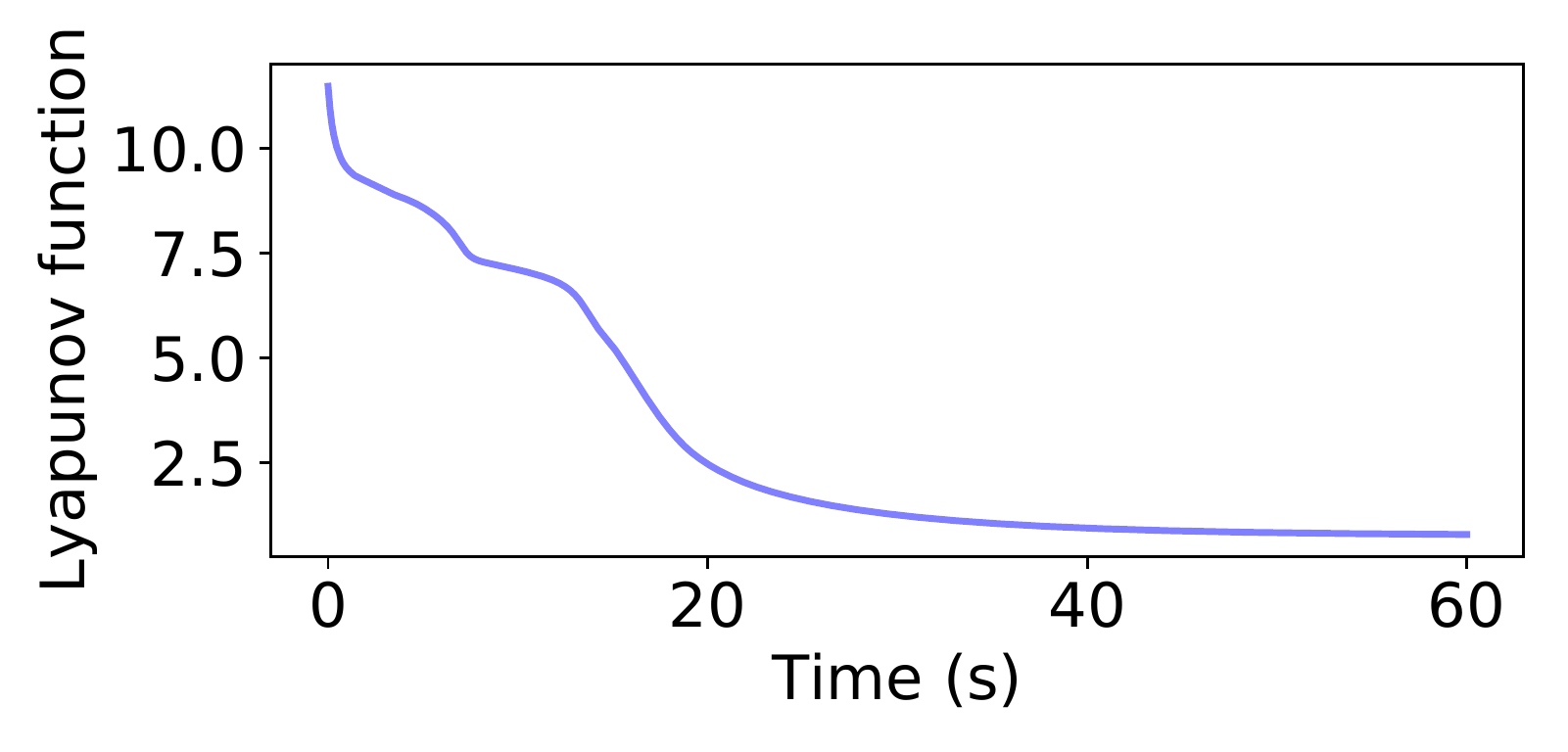}
	\end{subfigure}
	\begin{subfigure}[b]{0.19\textwidth}
		\centering
		\includegraphics[trim={0 0 0 0},clip,width=1\linewidth]{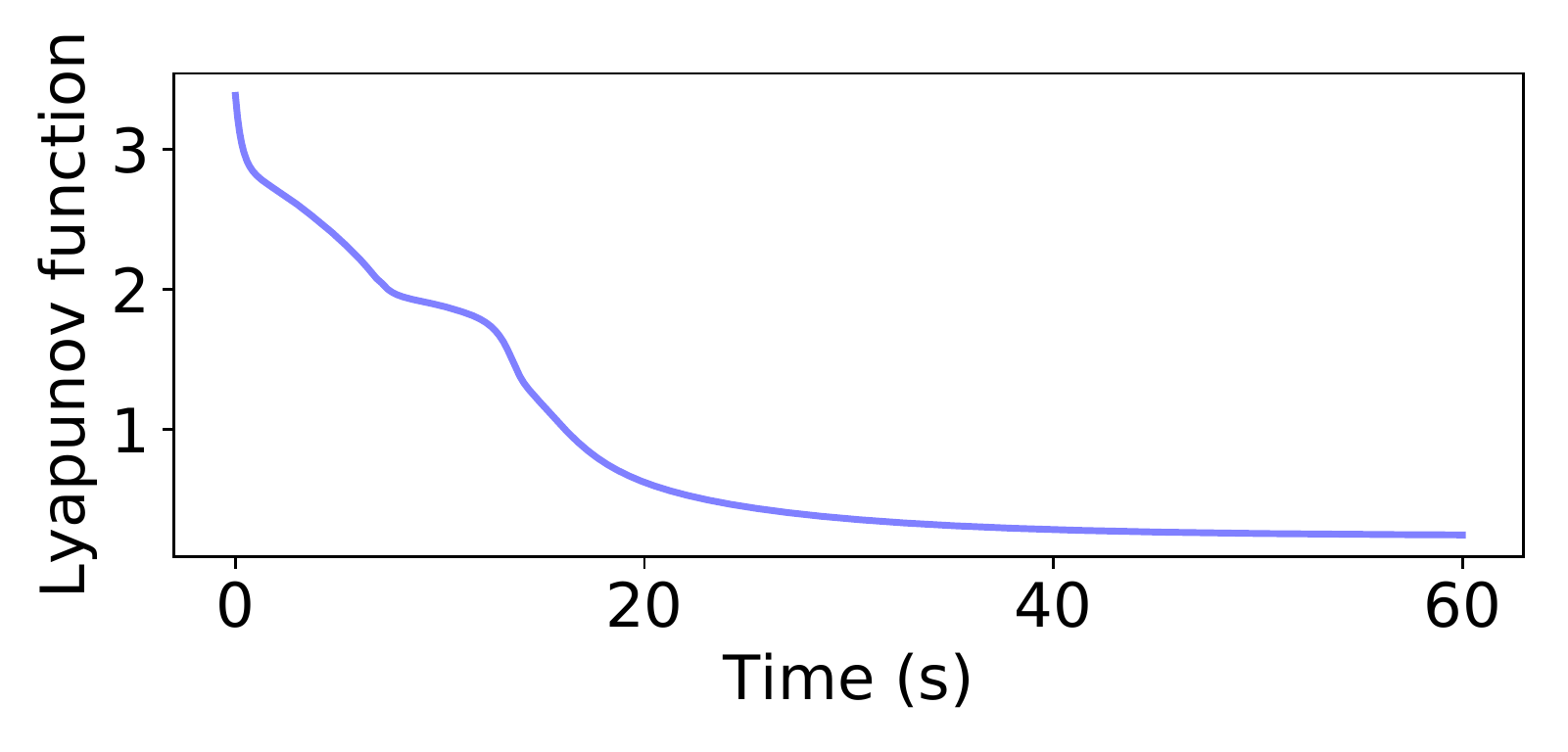}
	\end{subfigure}
	\begin{subfigure}[b]{0.19\textwidth}
		\centering
		\includegraphics[trim={0 0 0 0},clip,width=1\linewidth]{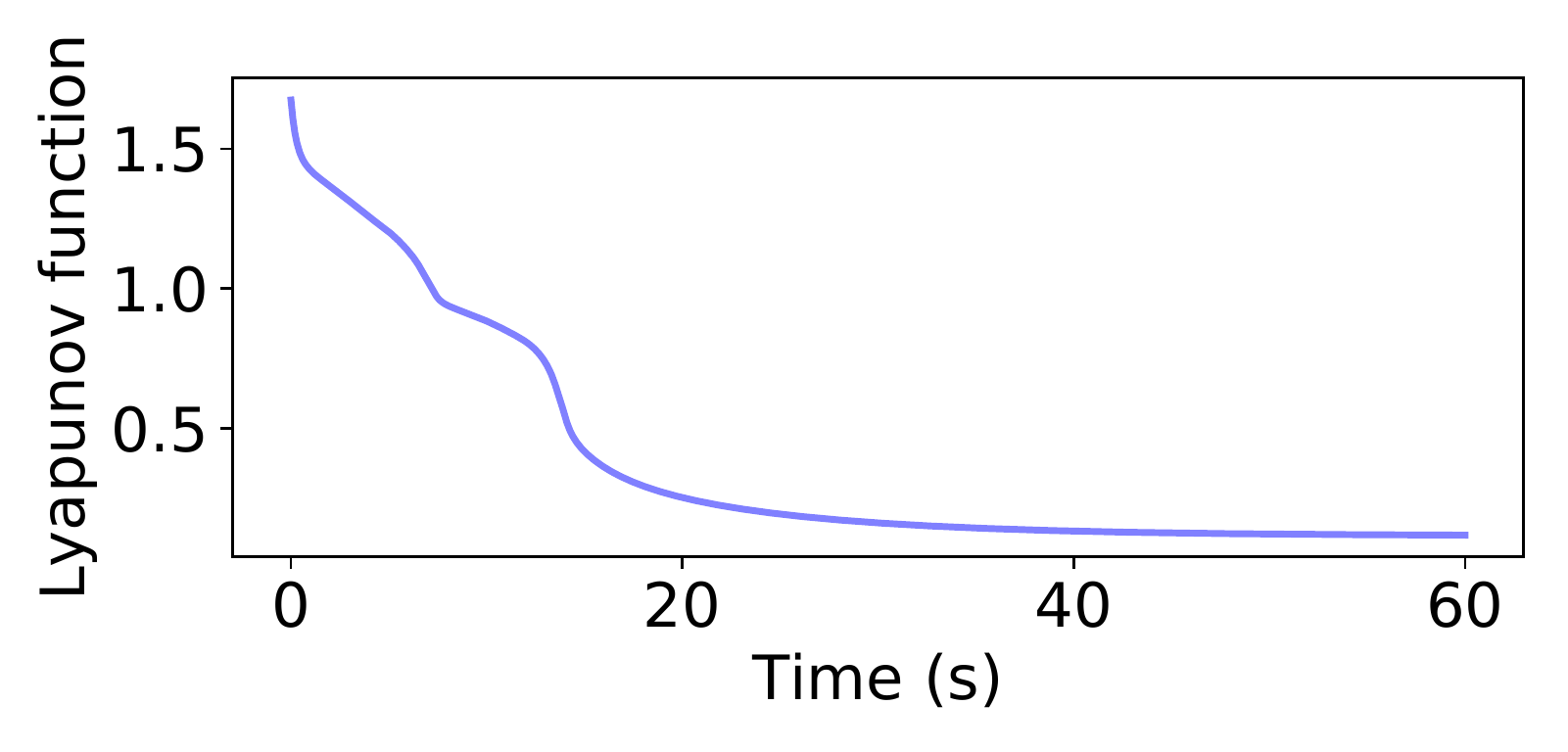}
	\end{subfigure}
	\begin{subfigure}[b]{0.19\textwidth}
		\centering
		\includegraphics[trim={0 0 0 0},clip,width=1\linewidth]{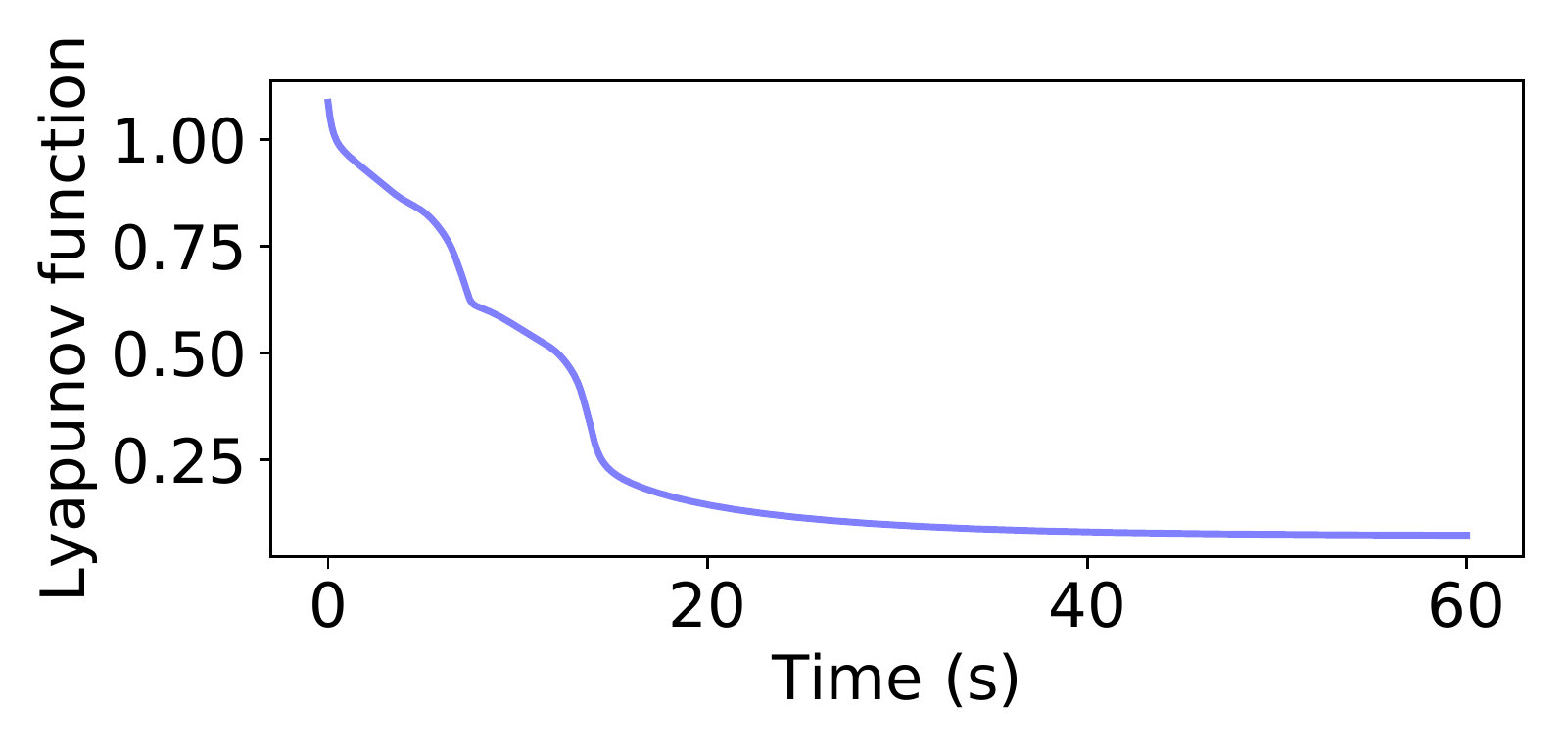}
	\end{subfigure}
	\vspace{-2mm}
	\caption{\small Improvement of the behavior produced by \texttt{learner-rmp} at various stages during training for \texttt{2d2level}. The top row shows the trajectories and the bottom row shows the corresponding Lyapunov function. From left to right these plots correspond to the red dots from left to right on the training curve in Figure~\ref{fig:2d2l_loss}.}
	\label{fig:2d_inc}
	\vspace{-6mm}
\end{figure}

\textbf{Results\quad}
We report two types of test loss: the batch-loss is the average loss on the entire test dataset generated by the expert policy, and the online-loss is the average loss at every time step (1 second interval) on the trajectories generated by the learner's policy starting from the initial states in the test dataset. In \texttt{2d1level}, the batch-loss is $5.42\times 10^{-5}$ and the online-loss is $5.82\times10^{-5}$. In \texttt{2d1level}, for \texttt{learner-rmp} the batch-loss is $2.45\times 10^{-4}$ and the online-loss is $2.78\times 10^{-4}$, while for \texttt{learner-un} the batch-loss is 0.111 and the online-loss is 12.203. 
The higher batch-loss for \texttt{learner-un} indicates that with the same amount of data and training the network is unable to learn the policy from the expert, while the much worse online-loss indicates that it cannot generalize well and succumbs to covariate shift problems.

Figures~\ref{fig:2d1l_traj},~\ref{fig:2d2l_traj} and~\ref{fig:2d2l_u_traj} show the evaluation of the trained networks on an example test environment. These results show that \newflow can perfectly match the behavior of the expert in the convex case (\texttt{2d1level}), while achieving near-expert performance in the non-identifiable case (\texttt{2d12evel}). From the overall results we also observe that \texttt{learner-un} is never able to reach the goal and also has a collision rate of 28\% (e.g. Figure~\ref{fig:2d2l_u_traj}), whereas \texttt{learner-rmp} successfully finishes the task 100\% of the time. We also tried a unstructured network with 5.8 times the number of learnable parameters. While the loss values improved with a small drop in collision rate, it was still never able to complete the task (please see Appendix~\ref{app:exp} for more details).
Figure~\ref{fig:2d_inc} shows the improving progression of \texttt{learner-rmp} during training, in which each snapshot corresponds to an associated point on the training curve in Figure~\ref{fig:2d2l_loss}. This verifies that with training we can progressively improve the behavior of the learner.
In addition, we verify that the stability properties of \newflow in the associated Lyapunov functions in Figure~\ref{fig:2d2l_energy} and Figure~\ref{fig:2d_inc}. We see that, regardless of the setting, the Lyapunov functions always decays monotonically as indicated by Theorem~\ref{th:newflow property}. This suggests \newflow produces a stable policy even when the learned weight functions are premature before the learning has converged (Figure~\ref{fig:2d_inc}). On the other hand, \texttt{learner-un} does not always avoid collision or provide any stability during or after training (see Figure~\ref{fig:2d_inc_unstruc} in Appendix~\ref{app:exp}).

\vspace{-4mm}
\subsection{Franka Robot} \label{sec:franka exps}
\vspace{-2mm}

We also validate our approach in a more realistic setup with a Franka Panda 7-DOF robot arm. In these experiments, the task is to reach a goal while navigating around an obstacle. The \newtree used is shown in Figure~\ref{fig:franka_tree}, where the configuration space of the robot is the root node, and weights functions are shown on the edges where they are defined. Please see Appendix~\ref{app:exp} for details.

\begin{figure}[!t]
	\centering
	\begin{subfigure}[b]{0.4\linewidth}
		\centering
		\includegraphics[trim={0 0 0 0},clip,width=0.48\linewidth]{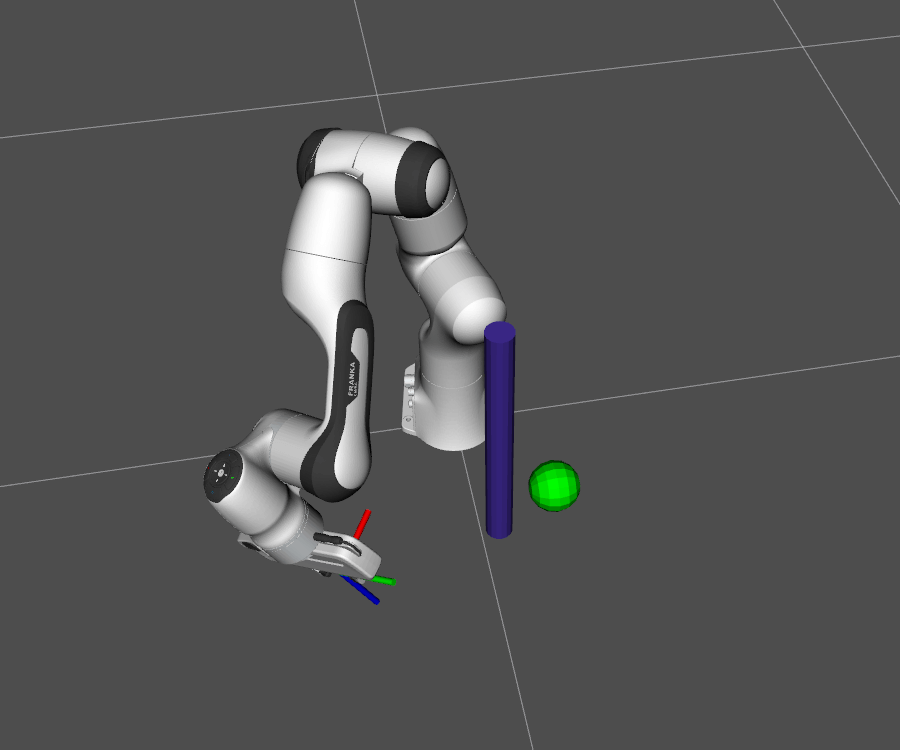}
		\includegraphics[trim={0 0 0 0},clip,width=0.48\linewidth]{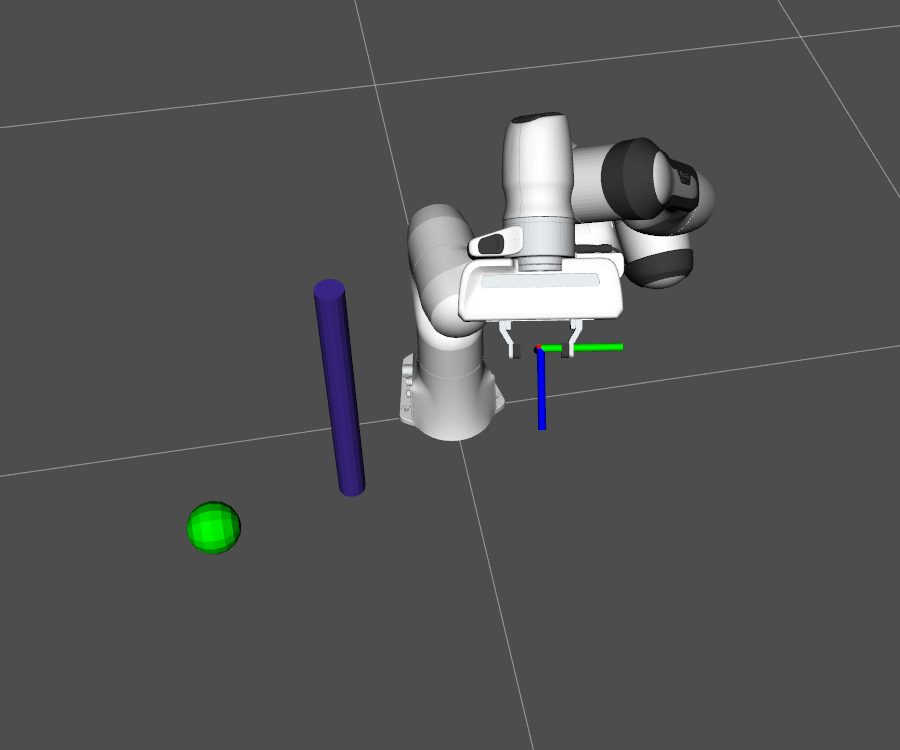}
		\caption{}
		\label{fig:franka_env}
	\end{subfigure}
	\begin{subfigure}[b]{0.55\linewidth}
		\centering
		\includegraphics[trim={80 0 50 0},clip,width=1\linewidth]{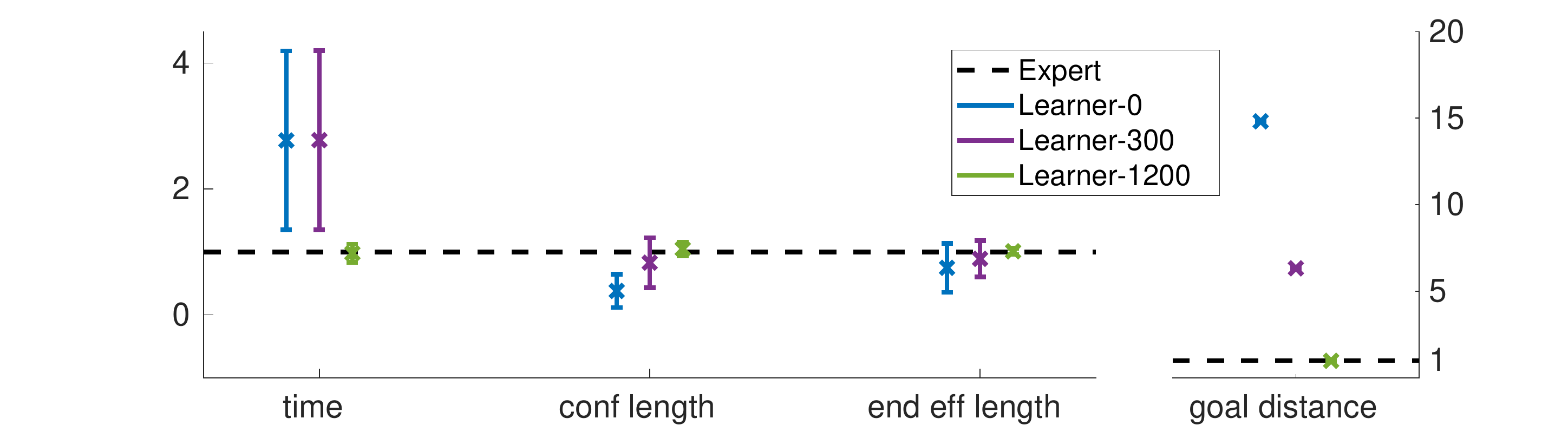}
		\caption{}
		\label{fig:franka_metrics}
	\end{subfigure}
	\vspace{-2mm}
	\caption{\small (a) An example from the training dataset (left) and the test dataset (right). The robot is shown in its start configuration with an obstacle (cylinder) and a goal (sphere). (b) Learner's performance with respect to the expert on the test dataset for the experiments with the Franka robot.}
	\vspace{-6mm}
\end{figure}

\textbf{Data and training\quad}
The expert policy is given by the \newtree with some fixed but unknown weights, while the learner's policy is defined by the \newtree with neural network weight functions that will be learned through behavior cloning. For training data we place an obstacle in a fixed location near the robot and sample different start configurations and goal locations that are in a region in front of the obstacle from the robot's perspective, so that the robot is forced to interact with the obstacle while trying to reach the goal. We run the expert to generate 110 unique trajectories for the training data. The trajectories are 5-10 seconds long and data is collected every 0.1 seconds; a data point consists of the state (configuration position and velocity of the robot), the auxiliary state (distances to goal and obstacle), and the expert action (acceleration). In a new environment with a different placement of the obstacle, this process is repeated to gather the test dataset where the expert is used to generate 20 unique trajectories. An example from the training and test dataset is shown in Figure~\ref{fig:franka_env}. The loss function is the same as in the experiments with the 2D robot and we train the policy using ADAM~\citep{kingma2014adam} with a minibatch size of 200 for 1500 iterations.
The number of iterations were chosen such that learning roughly converged and over-fitting had not happened.

\textbf{Results\quad}
We compare the performance of the learner, against the expert, at various stages of training: \texttt{learner-0} at no training (the neural network is initialized with random weights), \texttt{learner-300} at 300 iterations, and \texttt{learner-1200} at 1200 iterations when the learning converges. We record the following metrics on the test dataset for the expert and all the learners: (i) time: the time to reach within a precision of 0.05$m$ of the goal; we time-out the execution at 10 seconds, (ii) conf length: the distance traveled in configuration space, (iii) end eff length: the distance traveled by the end effector in workspace, and (iv) goal distance: the distance to the goal from the end effector at the end of an execution.

Figure~\ref{fig:franka_metrics} shows the performance of the learners relative to the expert on the test dataset (it plots the mean and the standard deviation of the ratios of the learner's metric and the expert's metric across trajectories; the expert is shown as the dotted horizontal baseline). From these results we see that, when the learner is not trained, the robot does not move much and incurs a high goal distance before timing out. With more training, the goal error reduces as the robot start traveling towards the goal but it still often times out. As the learning converges so does the performance of the learner towards the expert's performance. In all the trajectories across all the learners there are no collision, which verifies that constraints like safety can be incorporated through the structured learning approach that \newflow allows. We do a qualitative comparison on an example execution with the expert and the learners and also verify the stability properties of \newflow (even during learning) with the monotonically decreasing Lyapunov functions on these executions. Please see Appendix~\ref{app:exp} and Figure~\ref{fig:franka_traj} therein for details.

\vspace{-4mm}
\section{Conclusion}\label{sec:conc}
\vspace{-3mm}

We introduce extra parametrization flexibility into \flow and propose a new algorithm called \newflow. \newflow features a set of learnable weight functions that specifies the importance of subtask policies based on the robot's configuration and the environment.
Consequently, \newflow can combine imperfect subtask policies into a global policy with good performance, where the original \flow fails. We demonstrate the ability of \newflow to learn weight functions for policy fusion in experiments, and further theoretically prove that \newflow inherits the Lyapunov-type stability from \flow with only mild conditions on the weight functions. 
These structural properties and encouraging experimental results of \newflow suggest that \newflow can be treated as a class of structural policies suitable for policy learning with safety and interpretability requirements. 
Important future work includes designing more expressive policies based on \newflow, e.g., we can modify \newflow slightly to also learn part of the subtask policies and extra perturbations. 



\clearpage
\acknowledgments{
This research was partially supported by NSF CAREER award 1750483 and NSF NRI award 1637758.
}


\bibliography{refs}  

\clearpage
\appendix

\section*{Appendix}
\section{Proof of \cref{th:newflow property}} \label{app:proof}
We provide the proof of Theorem~\ref{th:newflow property} for completeness. We use~\cite[Thoerem 1]{cheng2018rmpflow} as the main lemma in our proof. 

\subsection{Background}
We first recall the definition of structured GDS~\cite{cheng2018rmpflow}, which augments a GDS with the information on how the metric matrix $\Gb$ factorizes,  in order to state~\cite[Thoerem 1]{cheng2018rmpflow}.
\begin{definition}\label{def:structured GDS} \cite{cheng2018rmpflow} 
	Suppose $\Gb$ has a structure $\SS$ that factorizes $\G(\x,\xd) = \J(\x)^\t \Hb(\z,\zd) \J(\x)$, where 
	$\z: \x \mapsto \z(\x) \in \R^n$  and $\Hb: \R^n \times \R^n \to \R^{n\times n}_+$, and $\J(\x) = \partial_\x \z$. 
	The tuple $(\MM, \G, \B, \Phi)_{\SS}$ is a \emph{structured GDS}, if 
	\begin{align} \label{eq:structured GDS}
	\Mb(\x,\xd)  \xdd 
	+ \bm\eta_{\G;\SS}(\x,\xd)  = - \nabla_\x \Phi(\x) - \Bb(\x,\xd)\xd 
	\end{align}
	where
	$
	\bm\eta_{\G;\SS}(\x,\xd) 
	\coloneqq  \J(\x)^\t ( \bm\xi_{\Hb}(\z,\zd) + 
	(\Hb(\z,\dot\z) + \bm\Xi_{\Hb}(\z,\zd) ) 
	\dot\J(\x,\xd) \xd  )
	$. 
	Given two structures, $\SS_a$ is said to \emph{preserve} $\SS_b$ if $\SS_a$ has the factorization (of $\Hb$) made by $\SS_b$.
\end{definition}
\noindent As noted in~\cite{cheng2018rmpflow}, GDSs are structured GDSs with a trivial structure (i.e. $\z =\x$), and structured GDSs reduce to GDSs if $\G(\x,\xd)=\G(\x)$, or if the manifold is one-dimensional.

\begin{lemma}	 \label{lm:consistency}
	\cite[Thoerem 1]{cheng2018rmpflow}
	Suppose the $i$th child node follows $(\NN_i, \G_i, \B_i, \Phi_i)_{\SS_i}$ and has coordinate $\y_i$. 
	Let $\ab_i = (\G_i + \bm\Xi_{\G_i})^{\dagger} (-\bm\eta_{\G_i;\SS_i} - \nabla_{\y_i} \Phi_i - \B_{i}\yd_i) $ and $\M_i =\G_i + \bm\Xi_{\G_i}$.
	Suppose $\ab$ of the parent node is given by \emph{\pullback} with $\{(\ab_i, \M_i)_C^{\NN_i} \}_{i=1}^K$. 
	Then $\ab$ 
	follows the \emph{pullback structured GDS}
	$(\MM, \G, \B, \Phi)_\SS$, 
	where $\Gb = \sum_{i=1}^{K}\J_i^\t\G_i\J_i$, $\B = \sum_{i=1}^{K}\J_i^\t \B_i \J_i$, $\Phi =  \sum_{i=1}^{K}\Phi_i \circ \y_k $, $\SS$ preserves $\SS_i$, and $\J_i = \partial_\x \y_i$.
	That is, the parent node is $(\ab, \M)_{C}^{\MM}$ such that  $\M = \sum_{i=1}^{K}\J_i^\t (\G_i+\bm\Xi_{\G_i} )\J_i$ and
	$
	\ab  = \left(\Gb + \bm\Xi_{\G}\right)^\dagger \left( 
	- \bm\eta_{\G;\SS}  - \nabla_\x \Phi - \Bb\xd  \right)
	$.
\end{lemma}

Lemma~\ref{lm:consistency} shows that the original \pullback operator preserves structured GDSs. Consequently, when all the leaf nodes are GDSs, the root node is a structured GDS, which implies the type of Lyapunov stability in Theorem~\ref{th:flow property}. 

\subsection{Proof of Theorem~\ref{th:newflow property}}

We prove the stability of \newflow using similar techniques as~\cite{cheng2018rmpflow}. Using the recursive property, it is sufficient to show that \newpullback preserves a family of structured GDSs, which are specified by the weight functions. Then the statement of Theorem~\ref{th:newflow property} follows directly as in~\cite{cheng2018rmpflow}.

We proceed by first decoupling the \newpullback into two steps. Let $u$ be a parent node on manifold $\MM$ and $\{v_k\}_{k=1}^K$  be its $K$ child nodes on manifold $\{\NN_k\}_{k=1}^K$ in an \newtree. Between  $u$ and each $v_k$, we add an extra node $\tilde{v}_k$ on manifold $\MM$ to create a new graph. In this new graph, $u$ has $K$ child nodes $\{\tilde{v}_k\}_{k=1}^K$ \emph{with identity transformation and the original weight function $w_k$}, and $\tilde{v}_k$ has a single child which is $v_k$ \emph{with the original transformation from $u$ to $v_k$ and an identity weight function}. 
Under this construction, the \newpullback operator in the original graph can then be realized in the new graph as 
\begin{enumerate}
	\item a \newpullback operator from $v_k$ to $\tilde{v}_k$ for each $k$ 
	\item a \newpullback operator from $\{\tilde{v}_k\}_{k=1}^K$ to $u$. 
\end{enumerate}
To verify this we rewrite~\eqref{eq:modifed pullback} as
\begin{align*} 
\begin{split}
\f &=\textstyle \sum_{i=1}^{K} w_i \J_i^\t (\f_{i} - \M_i \dot{\J}_i \xd) +  \hb_i \eqqcolon \sum_{i=1}^{K} w_i \tilde{\f}_i +  \tilde{\hb}_i\\
\M &=\textstyle \sum_{i=1}^{K} w_i \J_i^\t \M_i \J_i \eqqcolon \sum_{i=1}^{K} w_i \tilde{\M}_i\\
\G &=\textstyle \sum_{i=1}^{K} w_i \J_i^\t \G_i \J_i \eqqcolon \sum_{i=1}^{K} w_i \tilde{\G}_i\\
L &=\textstyle \sum_{i=1}^{K} w_i  L_i \eqqcolon \sum_{i=1}^{K} w_i  \tilde{L}_i
\end{split}	
\end{align*} 
where we also has
$
\hb_i = \tilde{\hb}_i =\tilde{L}_i \nabla_\x w_i-  (\xd^\t \nabla_\x w_i) \tilde\G_i \xd
$.
That is, node $\tilde{v}_k$ has the RMP $(\tilde{\f}_i, \tilde{\M}_i)_{C}^\MM$, the metric matrix $\tilde{\G}_i$, and the Lagrangian $\tilde{L}_i$. From the equalities above, we verify the two-step decomposition of \newpullback is valid.

Next we show that each step in the two-step decomposition yields a structured GDS like Lemma~\ref{lm:consistency}, which is sufficient condition we need to prove Theorem~\ref{th:newflow property}.  In the first step from $v_i$ to $\tilde{v}_i$, because the weight is constant identity, \newpullback is the same as \pullback. We apply Lemma~\ref{lm:consistency} and  conclude that  $\tilde{v}_i$ follows  $(\MM, \tilde\G_i, \tilde\B_i, \tilde\Phi_i)_{\tilde{\SS}_i}$, where
$\tilde{\SS}_i$ preserves $\SS_i$. 

Then we show the second step from $\{\tilde{v}_i\}_{i=1}^K$ to $u$ has similar properties. This is summarized as Lemma~\ref{lm:single newpullback} below.
\begin{lemma} \label{lm:single newpullback}
	If $\tilde{v}_i$ follows $(\MM, \tilde\G_i,\tilde\B_i, \tilde\Phi_i)_{\tilde{S}_i}$, then $u$ follows $
	(\MM, \G , \B, \Phi)_{S}
	$, where $S$ preserves $\tilde{S}_i$,  $\G = \sum_{i=1}^{K} w_i \tilde{\G}_i$, $\B = \sum_{i=1}^{K} w_i \tilde{\B}_i$, and $\Phi = \sum_{i=1}^{K} w_i \tilde{\Phi}_i$.
\end{lemma}

\begin{proof}[Proof of Lemma~\ref{lm:single newpullback}]
	This can be shown by algebraically comparing the dynamics of $(\MM, \G , \B, \Phi)_{S}$ and the result of~\eqref{eq:modifed pullback}.
	Let $\x$ be a coordinate of $\MM$ and, without loss of generality, let us consider $w_k$ to be a function of only $\x$ (we ignore the dependency on the auxiliary state). By Definition~\ref{def:structured GDS}, the dynamics of $(\MM, \G , \B, \Phi)_{S}$ satisfies
	\begin{align} \label{eq:resultant GDS}
	\M(\x,\xd) \xdd 
	+ {\bm\eta}_{\G;\SS}(\x,\xd)  = - \nabla_\x \Phi(\x) - \B(\x,\xd)\xd 
	\end{align}
	We first show the recursion of $\f$ of \newpullback satisfies~\eqref{eq:resultant GDS}. To this end, we rewrite ${\bm\eta}_{\G;\SS}$ by Definition~\ref{def:structured GDS} as
	\begin{align*}
	&\bm\eta_{\G;\SS}(\x,\xd)
	= \textstyle \sum_{i=1}^{K} \bm\xi_{w_i \tilde{\G}_i}(\x,\xd)  \\
	&= \textstyle \sum_{i=1}^{K} w_i(\x) \bm\eta_{\tilde{\G}_i}(\x,\xd) +  (\xd^\t \nabla_\x w_i(\x))\tilde{\G}_i(\x,\xd) \xd
	 \textstyle- \frac{1}{2}\nabla_\x w_i(\x) \xd^\t  \tilde{\G}_i(\x,\xd)\xd
	\end{align*}
	where in the first equality we use 
	the trick we made that the transformation from $u$ to $\tilde{v}_k$ is identity and we use  the fact $\tilde{S}_i$ preserves $S_i$, so the structure $S$ that preserves $\tilde{S}_i$  has a clean structure
	\begin{align*}
	\G = \begin{bmatrix}
	I &\dots &I
	\end{bmatrix}
	\begin{bmatrix}
	w_1 \tilde{\G}_1 & & \\
	& \ddots & \\
	& & w_K \tilde{\G}_K
	\end{bmatrix}
	\begin{bmatrix}
	I \\ \vdots \\I
	\end{bmatrix}
	\end{align*}
	Similarly, we rewrite $
	\nabla_\x \Phi(\x) = \textstyle
	\sum_{i=1}^K w_i(\x)\nabla_\x \tilde\Phi_i(\x) + \tilde{\Phi}_i \nabla_\x w_i(\x)
	$. Substituting these two equalities into~\eqref{eq:resultant GDS}, we can write (with input dependency omitted) 
	\begin{align*}
	&\M \xdd = - \nabla_\x \Phi - \B\xd  - {\bm\eta}_{\G;\SS}\\
	&=\textstyle \sum_{i=1}^{K} - w_i\nabla_\x \tilde\Phi_i - \tilde{\Phi}_i \nabla_\x w_i  - w_i\B_i\xd   + \sum_{i=1}^{K}- w_i \bm\eta_{\tilde{\G}_i} - (\xd^\t \nabla_\x w_i)\tilde{\G}_i \xd
	+  \frac{1}{2}\nabla_\x w_i\xd^\t  \tilde{\G}_i\xd 
	\\
	&=\textstyle \sum_{i=1}^{K}  w_i \tilde{\f}_i 
	+  \frac{1}{2}\nabla_\x w_i\xd^\t  \tilde{\G}_i\xd- \tilde{\Phi}_i \nabla_\x w_i - (\xd^\t \nabla_\x w_i)\tilde{\G}_i \xd
	\\
	&=\textstyle \sum_{i=1}^{K}  w_i \tilde{\f}_i + \hb_i
	\end{align*}
	where we use the fact that $\tilde{\f_i } = - \nabla_\x \tilde\Phi_i - \tilde\B_i\xd  - {\bm\eta}_{\tilde\G_i;\tilde\SS_i}$ as  $\tilde{v}_i$ follows  $(\MM, \tilde\G_i, \tilde\B_i, \tilde\Phi_i)_{\tilde{\SS}_i}$ with $\tilde{\SS}_i$ preserving $\SS_i$. This is exactly the recursion of $\f$ when \newpullback is applied between $\tilde{v}_i$ and $u$,  i.e. $\f = \M \xdd = \sum_{i=1}^{K}  w_i \tilde{\f}_i + \hb_i$. 
	
	To establish the equivalence of the other recursions, we next rewrite $\Mb$ by definition in~\eqref{eq:GDS} as
	\begin{align*}
	\Mb(\x,\xd) &= \G(\x,\xd) + \bm\Xi_{\G}(\x,\xd)\\
	&=\textstyle \sum_{i=1}^{K}w_i(\x)\left( \tilde{\G}_i(\x,\xd)  + \bm\Xi_{\tilde{\G}_i}(\x,\xd)\right)\\
	&=\textstyle \sum_{i=1}^{K}w_i(\x) \tilde{\Mb}_i(\x,\xd)
	\end{align*}
	where we use the fact that $w_i$ does not on the velocity $\xd$. The recursion for $\G$ and $L$ can be derived similarly, so we omit them here.
\end{proof}

So far we have shown that \newpullback of \newflow retains the closure of structured GDSs as \pullback in \flow. In addition, we show that the structured GDS created by \newpullback has a linearly weighted metric matrix, damping matrix, and potential function (cf. Lemma~\ref{lm:single newpullback}). By recursively applying the two-step decomposition above, from the leaf nodes to the root node, we conclude that the root node policy will be a structured GDS with a Lyapunov function given by the recursion in~\eqref{eq:recursive law}. 
The rest of the statement of Theorem~\ref{th:newflow property} follows from the properties of structured GDSs as shown in~\cite{cheng2018rmpflow}. 

\section{Benefits due to the Extra Flexibility of \newflow} \label{app:flexibility example}

We use an example to illustrate the extra flexibility offered by \newflow. 
Consider a simple Y-shape \newtree with a root node and two child nodes with weight functions $w_1$ and $w_2$. 
For the child nodes, suppose they are GDS $(\NN_i, \G_i, \B_i, \Phi_i)$ and have coordinate $\y_i$, for $i=1,2$. For simplicity, let us assume $\G_i$ only depends on the configuration $\y_i$.
From Theorem~\ref{th:newflow property}, we see that the root node has an energy function $V_{\texttt{r}} = \frac{1}{2} \qd^\t \G_r \qd + \Phi_r$, where
$\G_r(\q) =  w_1(\q) \G_1(\y_1(\q)) + w_2(\q) \G_2(\y_2(\q))$ and $\Phi_r(\q)  = w_1(\q)  \Phi_1(\y_1(\q)) + w_2 (\q) \Phi_2(\y_2(\q))  $
Because $w_i$ is a function of $\q$ not $\y_i$ and the Lyapunov function of \flow only allows summing child-node functions, this example root node policy does not admit a tree structure decomposition in the original \tree and can only be implemented as a single large node. Conversely, because of the weight function on the edges, \newtree can further exploits potential sparsity inside the policy representation 
so that building complicated global polices with only basic elementary policies becomes possible.

We note that the example above does not imply that \newflow can generate more expressive policies than \flow. More precisely, \newtree allows representing the same global policy using more basic leaf-node policies. This property has two implications: it suggests (i) \newflow can be more efficient to compute and (ii) \newflow can offload the difficulties of designing leaf-nodes policies into the weight functions, which are learnable.

\section{Learning \newflow} \label{app:learning}

To show the weights are learnable, it is sufficient to check  if we can  differentiate through the output of the final policy $\pi = \ab_{\texttt{r}}$ with respect to the parameters that specify the weight functions. As the computation of $\ab_{\texttt{r}}$ is accomplished recursively in the backward pass using \newpullback, we will only illustrate that  \newpullback is differentiable. This can be seen by treating \newpullback as a computation graph, as illustrated in Figure~\ref{fig:an example weight function}. 
Take the nodes in~\eqref{eq:modifed pullback} as an example. \newpullback receives $\f_i$, $\M_i$, $\G_i$, $\B_i$, $\J_i$, $\dot\J_i$, $L_i$ from the edges to the child nodes, the current state $(\x,\xd)$ and the auxiliary state to define the weight function $w_i$ and the correction term $\h_i$. As these inputs values do not depend on the weight functions $\{w_i\}$ at the current node (i.e. they do not form a loop), the derivative of $\ab_r$ with respect to the weight functions in the \newtree can be computed recursively by back-propagating the derivatives through each \newpullback operator. 

\section{Experimental Details}\label{app:exp}

\subsection{2D Robot}
\texttt{2d1level} consists of a 2D particle that aims to reach a goal while avoiding an obstacle. The \newtree for \texttt{2d1level} is of depth one (see Figure~\ref{fig:2d1level}), where the root node q (configuration space of the robot) has one child obstacle RMP node (o\textsubscript{rmp}) and one child attractor RMP node (a\textsubscript{rmp}). \texttt{2d2level} consists of a 2D particle that aims to reach a goal while avoiding two obstacles. The \newtree for \texttt{2d2level} is of depth two (see Fig~\ref{fig:2d2level}), where the root node (q) has one child attractor RMP node (a\textsubscript{rmp}) and one all-obstacle RMP (o) that is meant to combine two child obstacle RMPs (o\textsubscript{rmp}, one for each obstacle). The respective weight functions are shown on the edges of both these trees. The tree structures here are heuristically chosen based on the problem domain, as in RMPflow and typically follow the robot’s kinematic chain and then extend into the workspace and abstract task spaces.

Figure~\ref{fig:2d_inc_unstruc} shows the progression of \texttt{learner-un} during training, in which each snapshot corresponds to an associated point on the training curve in Figure~\ref{fig:2d2l_u_loss}. We see that \texttt{learner-un} is never able to reach the goal and often ends up in collision during and after training. 
We also compared with a unstructured network, \texttt{learner-un-large}, that has 5.8 times more learnable parameters compared to \texttt{learner-un}. We see improvement over loss values where the batch-loss is 0.065 and the online-loss is 0.393, and the collision rate decreases to 16\%. However, it is still never able to complete the task (e.g. see Figure~\ref{fig:2d2l_u_traj_large}). Figure~\ref{fig:2d_inc_unstruc_large} shows the progression of \texttt{learner-un-large} during training, in which each snapshot corresponds to an associated point on the training curve in Figure~\ref{fig:2d2l_u_loss_large}.

\begin{figure}[!t]
	\centering
	\begin{subfigure}[b]{0.25\textwidth}
		\centering
		\includegraphics[trim={0 0 0 0},clip,width=1\linewidth]{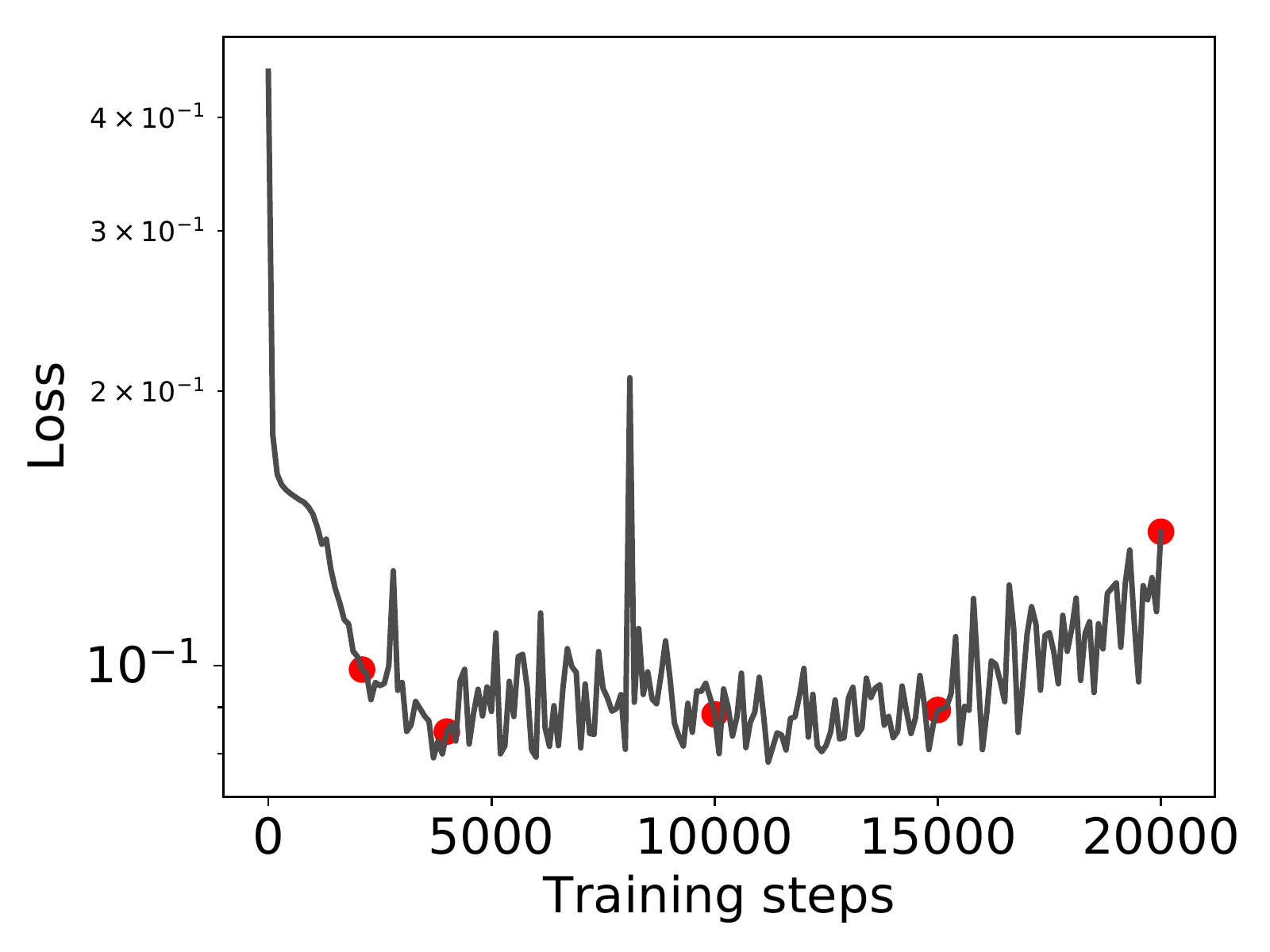}
		\caption{}
		\label{fig:2d2l_u_loss}
	\end{subfigure}
	\begin{subfigure}[b]{0.18\textwidth}
		\centering
		\includegraphics[trim={70 0 70 0},clip,width=1\linewidth]{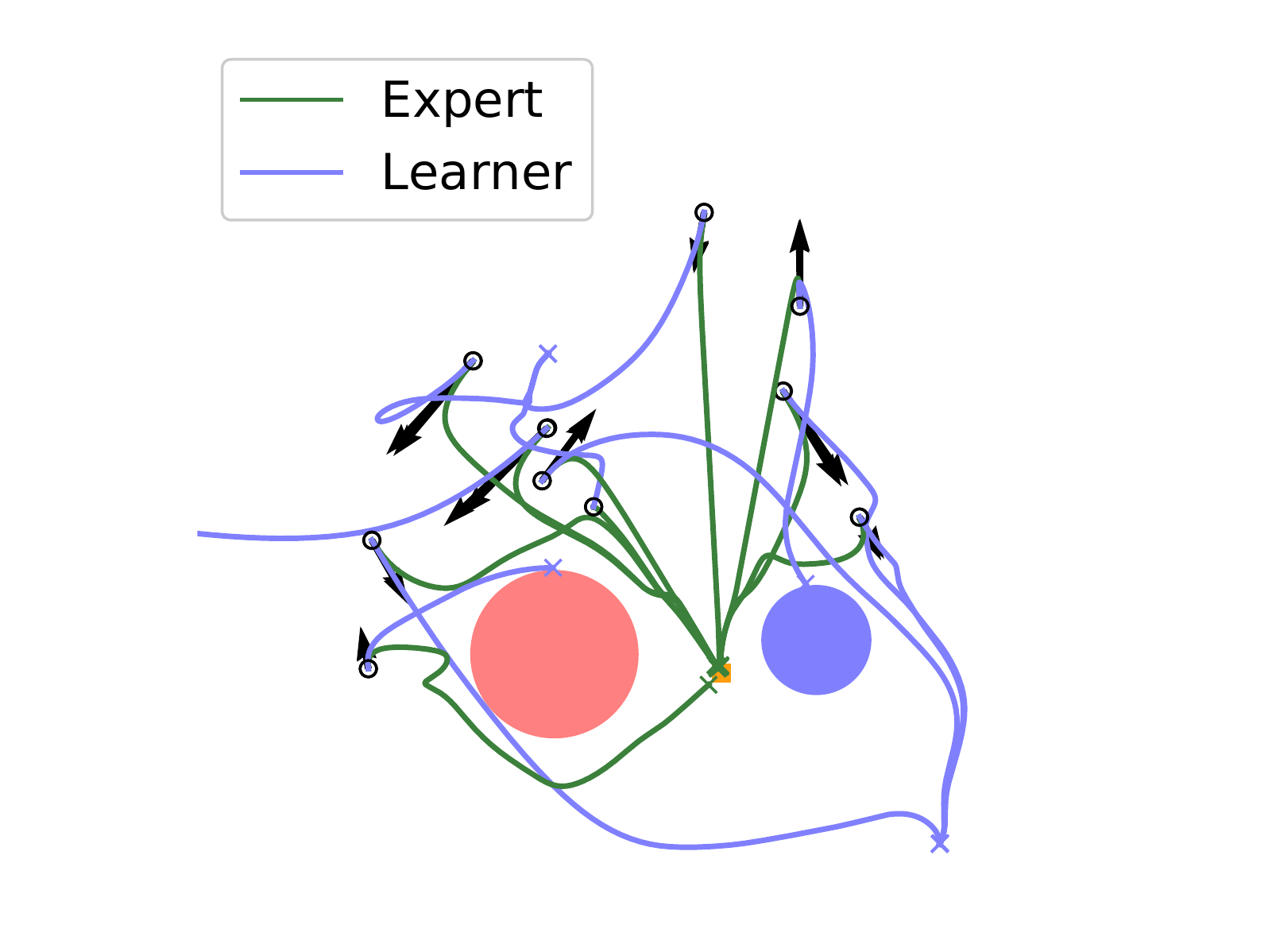}
		\caption{}
		\label{fig:2d2l_u_traj_large}
	\end{subfigure}
	\begin{subfigure}[b]{0.25\textwidth}
		\centering
		\includegraphics[trim={0 0 0 0},clip,width=1\linewidth]{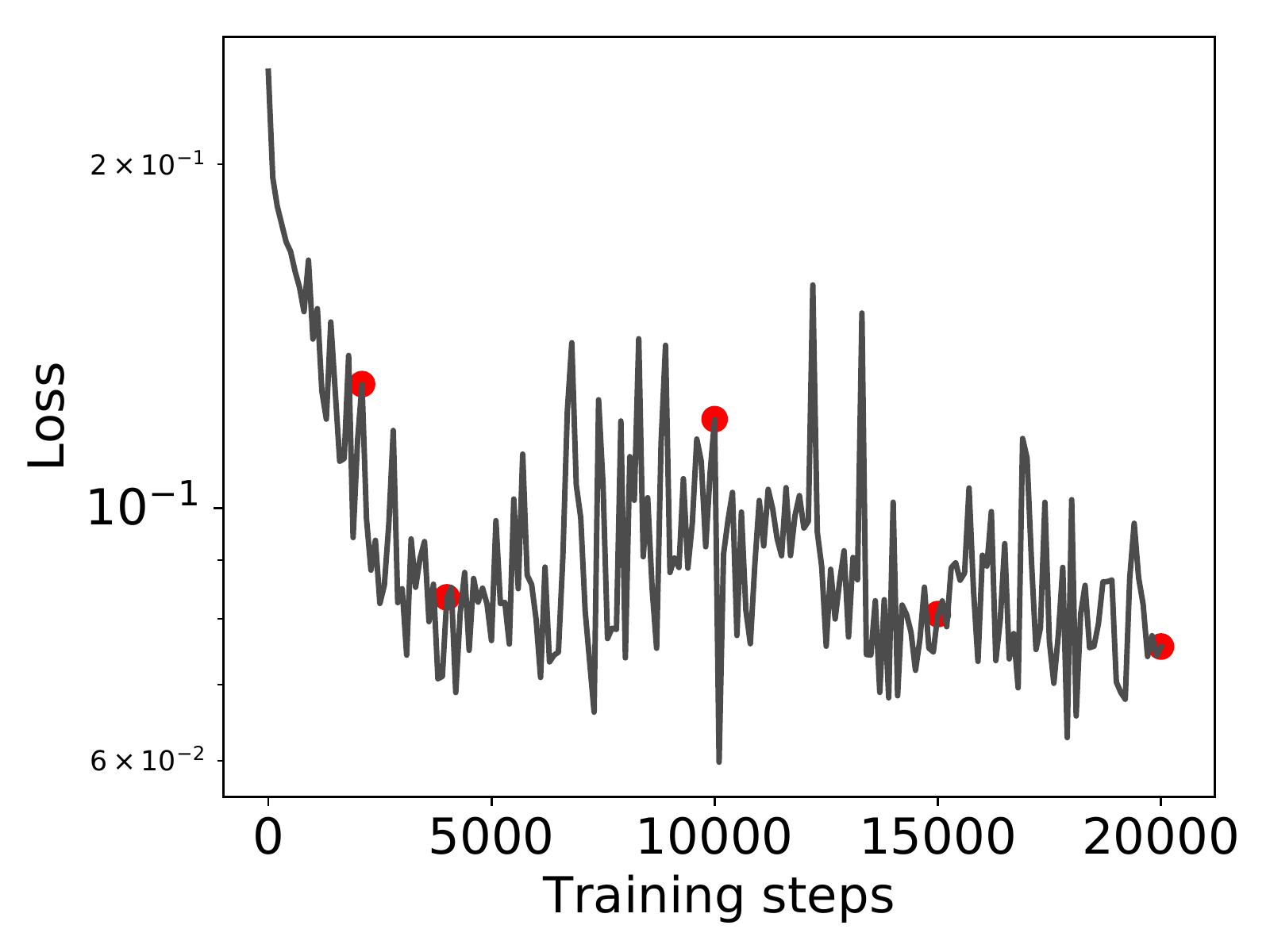}
		\caption{}
		\label{fig:2d2l_u_loss_large}
	\end{subfigure}
	\caption{\small (b) Trajectories generated in \texttt{2d2level} by \texttt{learner-rmp-large} compared to the expert is shown. Initial state is a black circle for position and black arrow for velocity. The environment has obstacles (red and blue) and goal (orange square). Learning curves for (a) \texttt{learner-rmp} and (c) \texttt{learner-rmp-large} on \texttt{2d2level} is also shown.}
\end{figure}

\begin{figure}[!t]
	\centering
	\begin{subfigure}[b]{0.19\textwidth}
		\centering
		\includegraphics[trim={100 50 90 100},clip,width=1\linewidth]{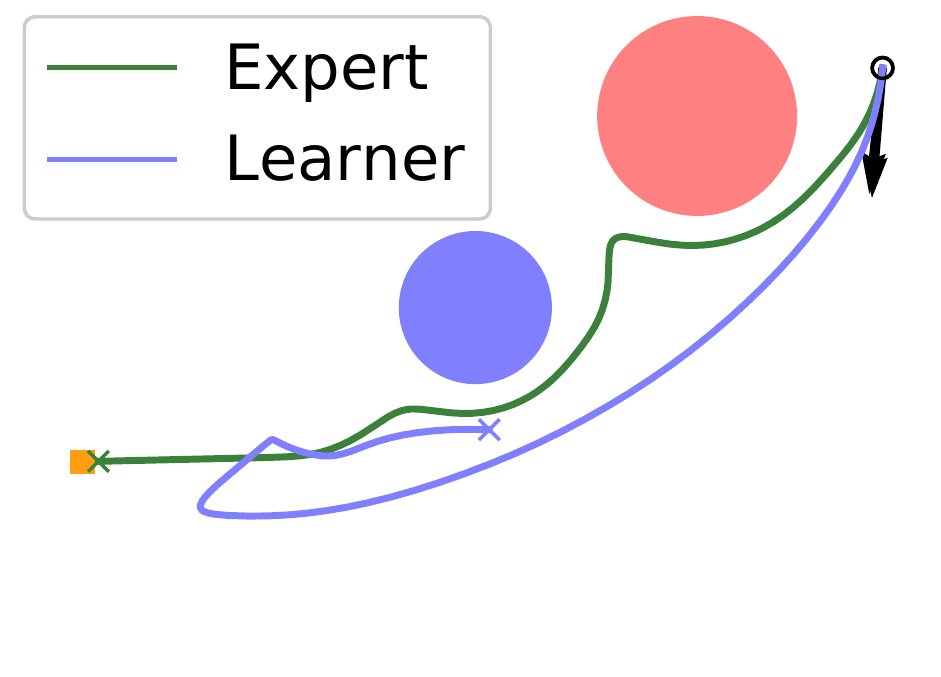}
	\end{subfigure}
	\begin{subfigure}[b]{0.19\textwidth}
		\centering
		\includegraphics[trim={100 50 90 100},clip,width=1\linewidth]{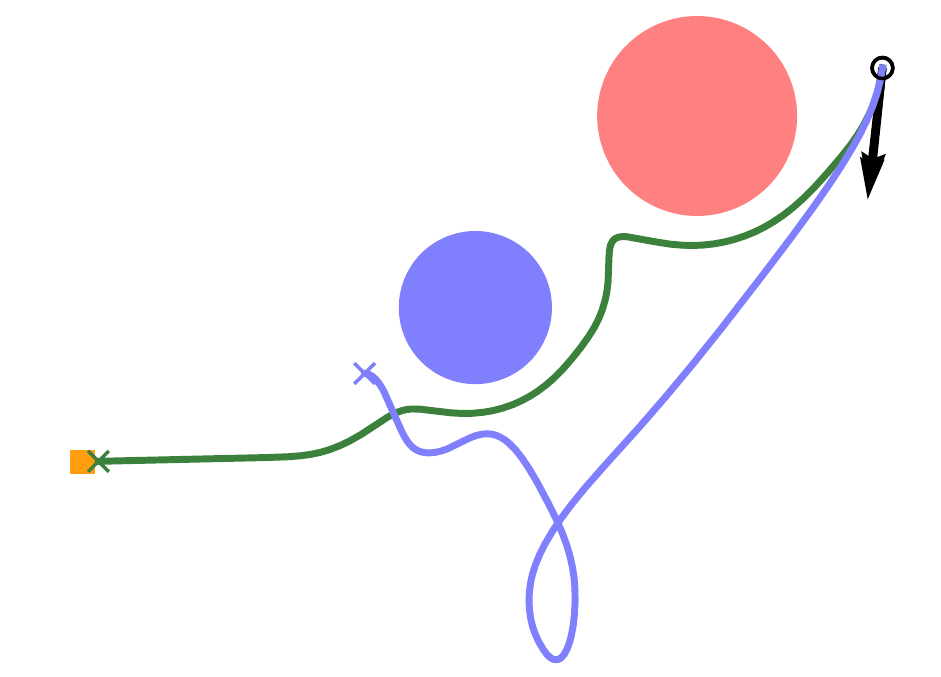}
	\end{subfigure}
	\begin{subfigure}[b]{0.19\textwidth}
		\centering
		\includegraphics[trim={100 50 90 100},clip,width=1\linewidth]{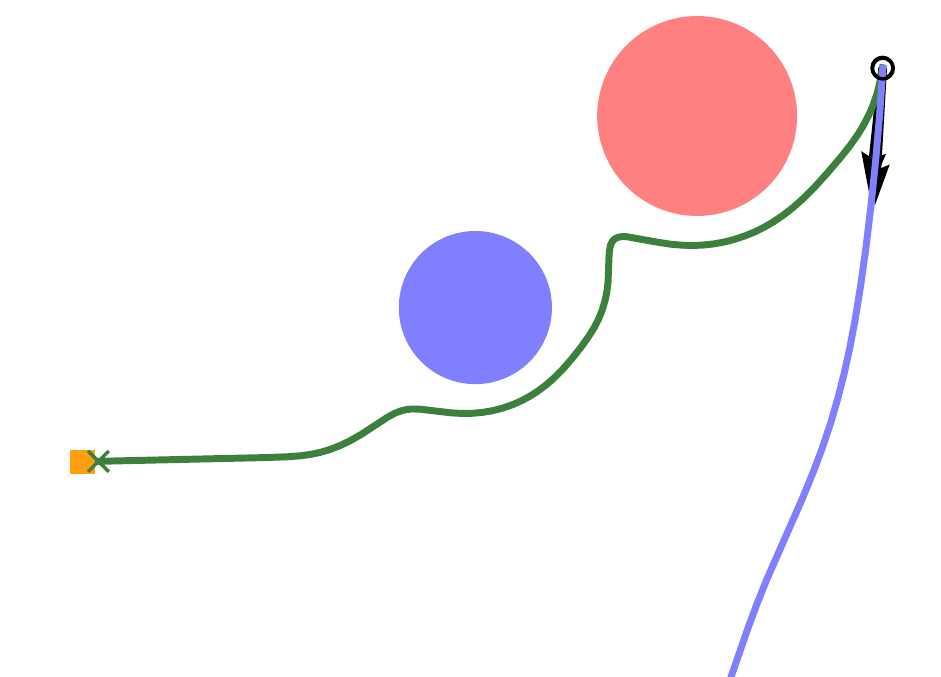}
	\end{subfigure}
	\begin{subfigure}[b]{0.19\textwidth}
		\centering
		\includegraphics[trim={100 50 90 100},clip,width=1\linewidth]{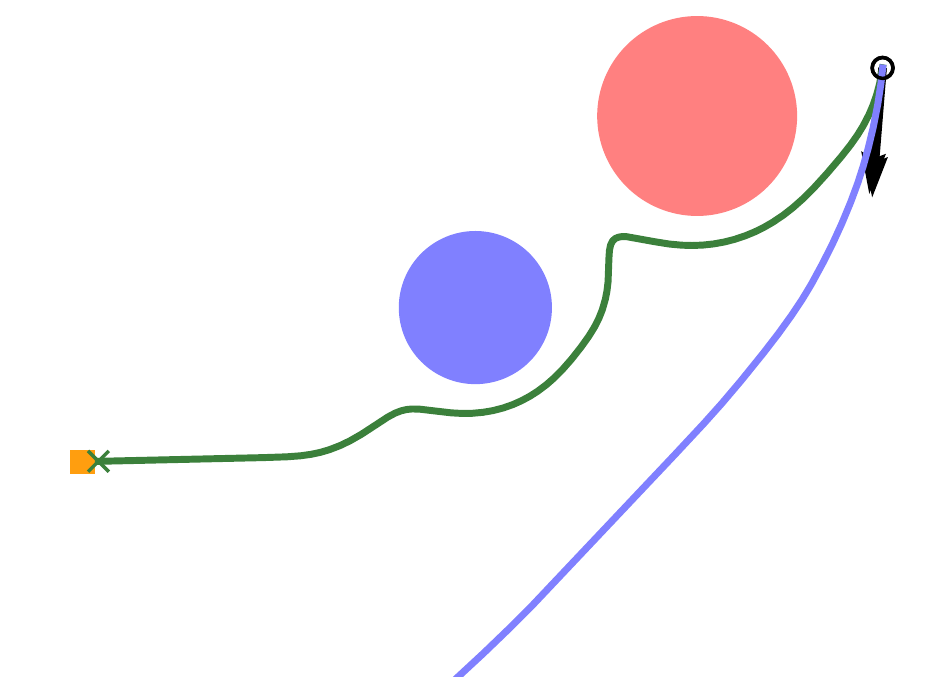}
	\end{subfigure}
	\begin{subfigure}[b]{0.19\textwidth}
		\centering
		\includegraphics[trim={100 50 90 100},clip,width=1\linewidth]{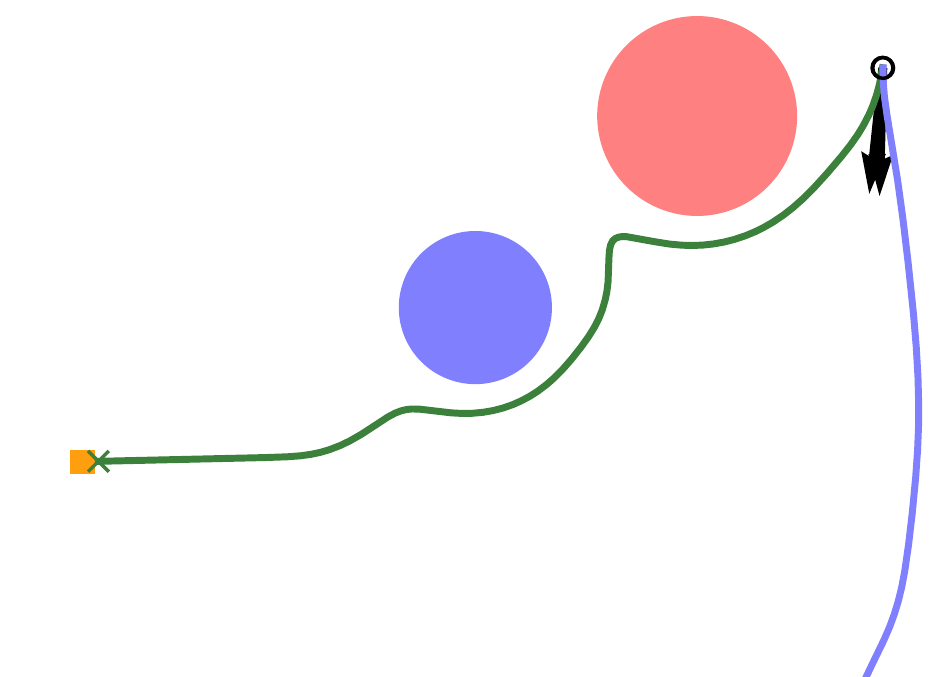}
	\end{subfigure}
	\caption{\small Trajectories produced by \texttt{learner-un} at various stages during training for \texttt{2d2level}. From left to right these plots correspond to the red dots from left to right on the training curve in Figure~\ref{fig:2d2l_u_loss}.}
	\label{fig:2d_inc_unstruc}
\end{figure}

\begin{figure}[!t]
	\centering
	\begin{subfigure}[b]{0.19\textwidth}
		\centering
		\includegraphics[trim={100 20 90 100},clip,width=1\linewidth]{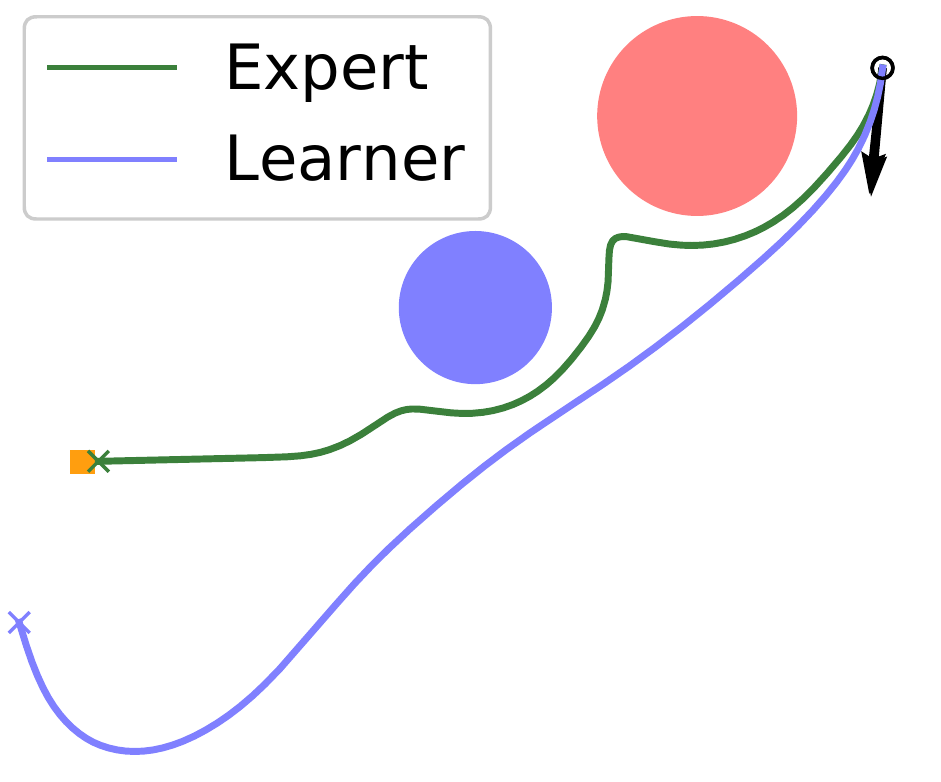}
	\end{subfigure}
	\begin{subfigure}[b]{0.19\textwidth}
		\centering
		\includegraphics[trim={100 20 90 100},clip,width=1\linewidth]{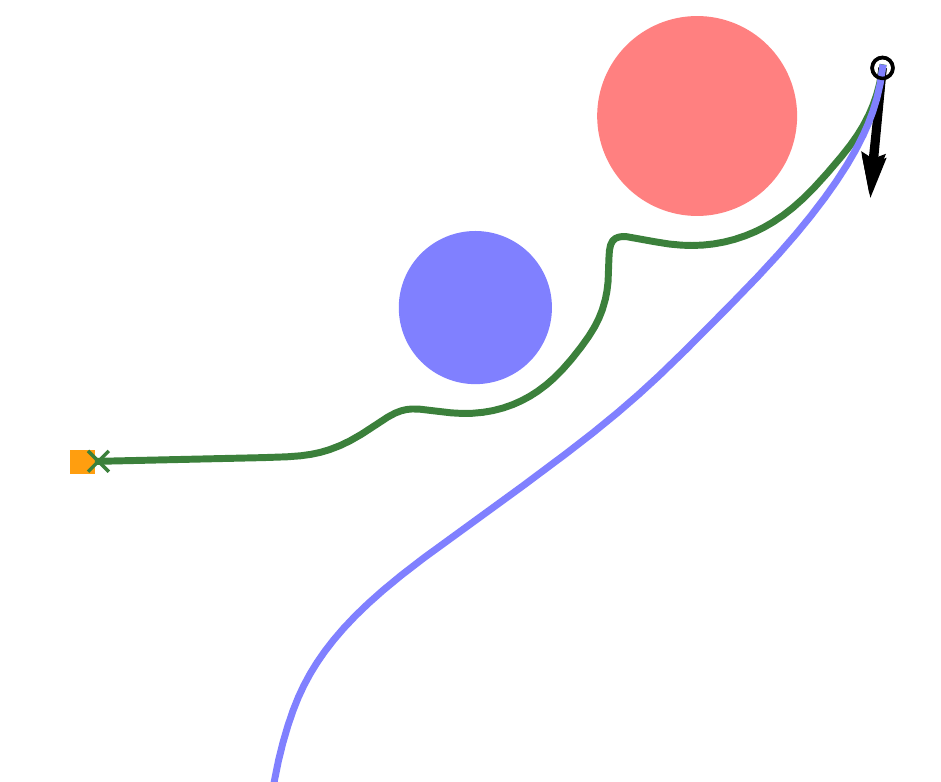}
	\end{subfigure}
	\begin{subfigure}[b]{0.19\textwidth}
		\centering
		\includegraphics[trim={100 20 90 100},clip,width=1\linewidth]{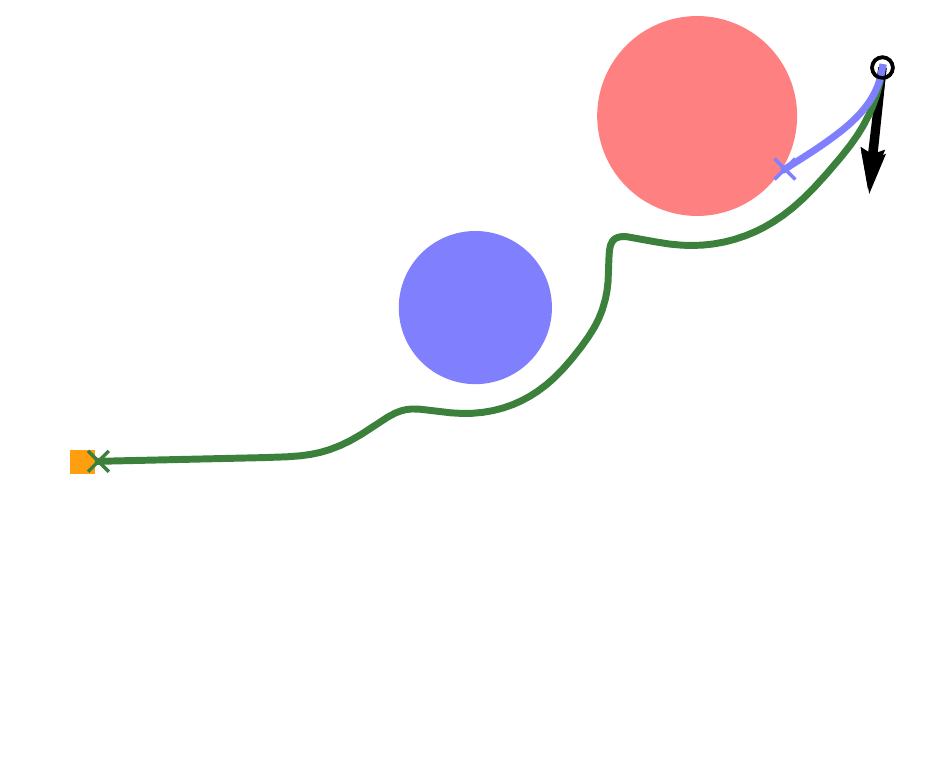}
	\end{subfigure}
	\begin{subfigure}[b]{0.19\textwidth}
		\centering
		\includegraphics[trim={100 20 90 100},clip,width=1\linewidth]{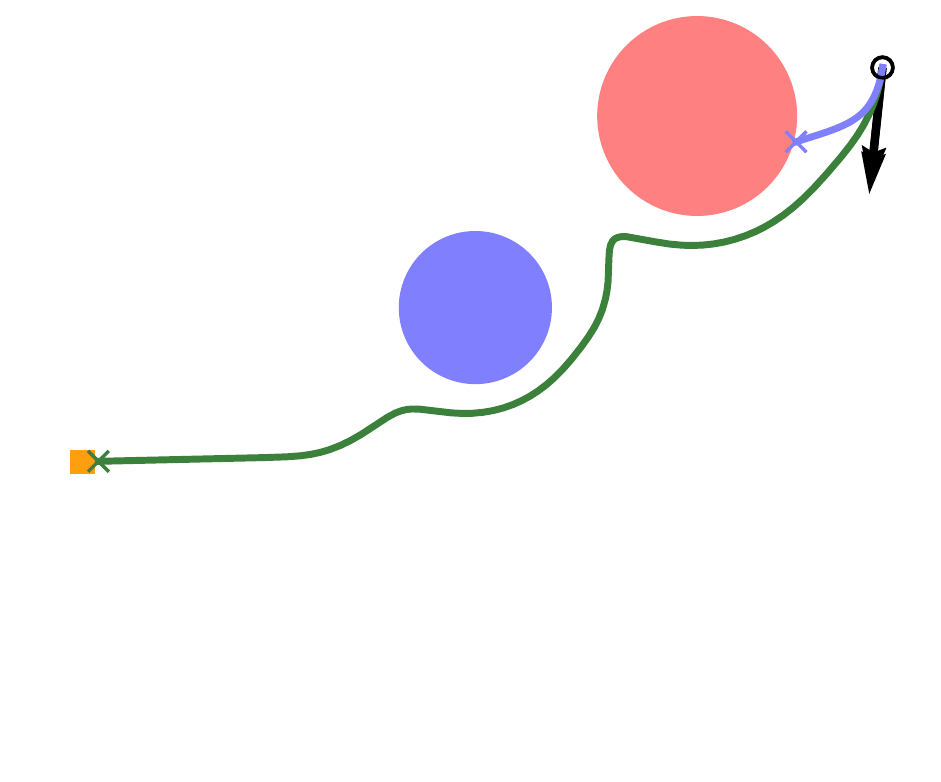}
	\end{subfigure}
	\begin{subfigure}[b]{0.19\textwidth}
		\centering
		\includegraphics[trim={100 20 90 100},clip,width=1\linewidth]{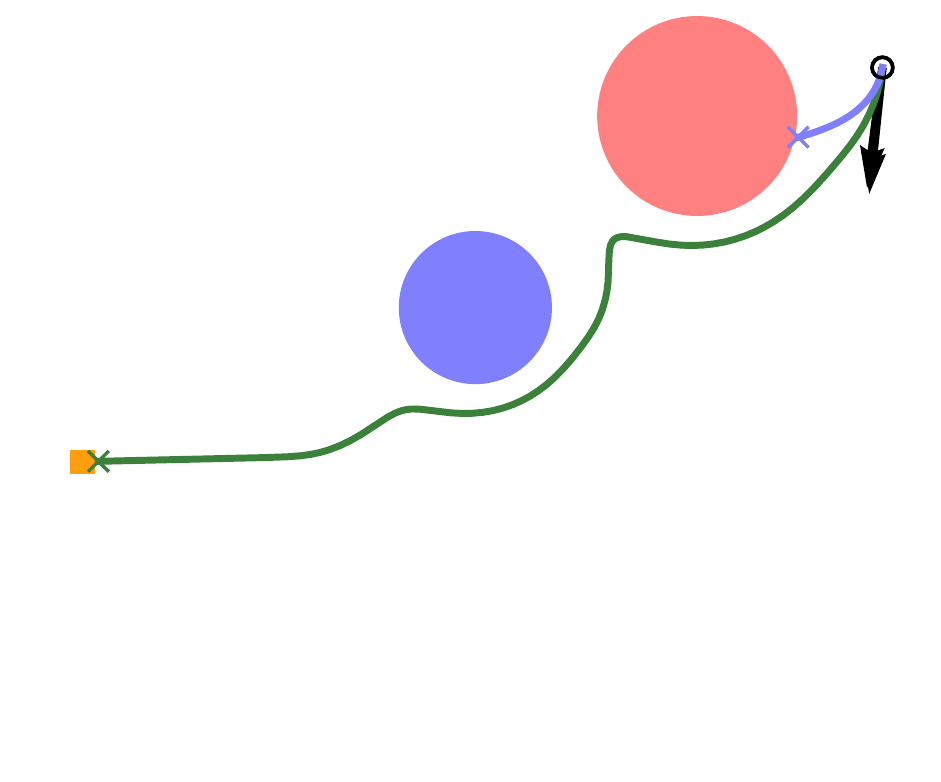}
	\end{subfigure}
	\caption{\small Trajectories produced by \texttt{learner-un-large} at various stages during training for \texttt{2d2level}. From left to right these plots correspond to the red dots from left to right on the training curve in Figure~\ref{fig:2d2l_u_loss_large}.}
	\label{fig:2d_inc_unstruc_large}
\end{figure}

\begin{figure}[h]
	\centering
	\begin{subfigure}[b]{1\linewidth}
		\centering
		\includegraphics[trim={97.5 0 97.5 0},clip,width=0.19\linewidth]{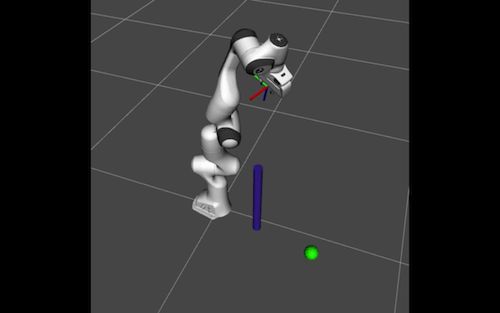}
		\includegraphics[trim={97.5 0 97.5 0},clip,width=0.19\linewidth]{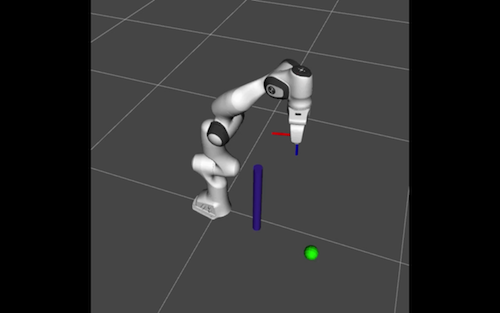}
		\includegraphics[trim={97.5 0 97.5 0},clip,width=0.19\linewidth]{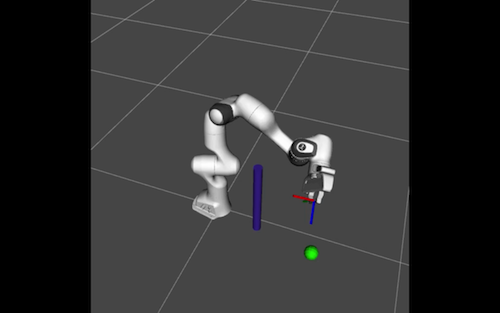}
		\includegraphics[trim={97.5 0 97.5 0},clip,width=0.19\linewidth]{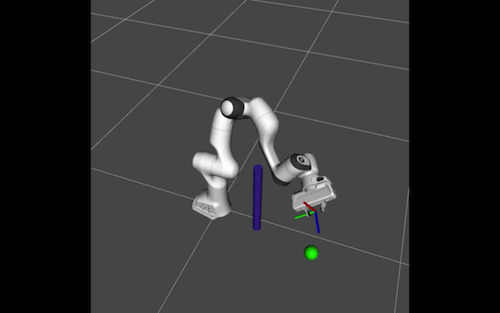}
		\includegraphics[trim={97.5 0 97.5 0},clip,width=0.19\linewidth]{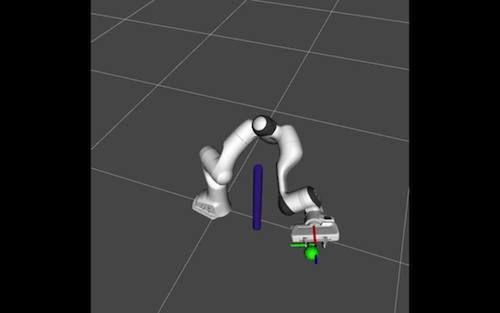}
		\caption{expert}
	\end{subfigure}
	\begin{subfigure}[b]{1\linewidth}
		\centering
		\includegraphics[trim={97.5 0 97.5 0},clip,width=0.19\linewidth]{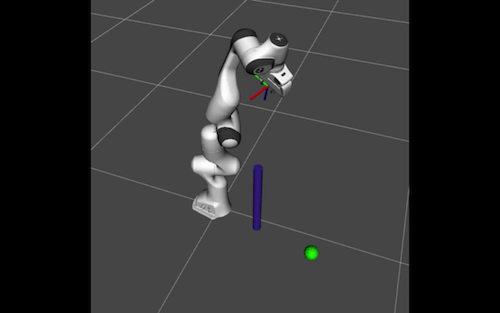}
		\includegraphics[trim={97.5 0 97.5 0},clip,width=0.19\linewidth]{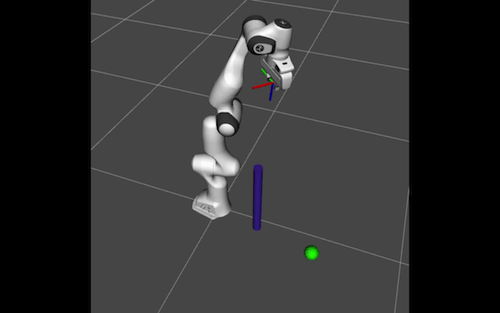}
		\includegraphics[trim={97.5 0 97.5 0},clip,width=0.19\linewidth]{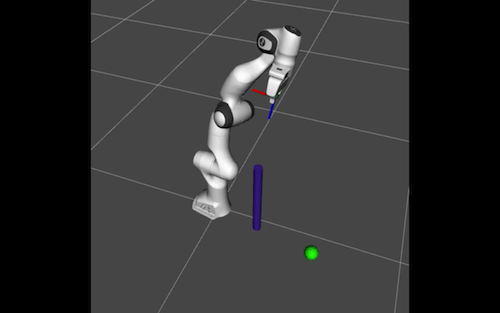}
		\includegraphics[trim={97.5 0 97.5 0},clip,width=0.19\linewidth]{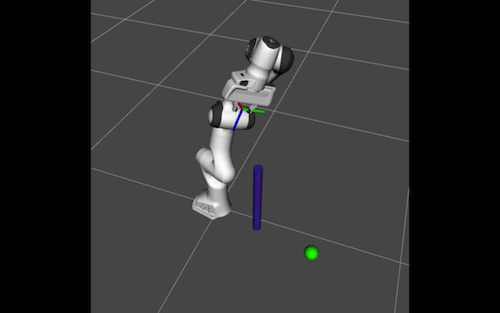}
		\includegraphics[trim={97.5 0 97.5 0},clip,width=0.19\linewidth]{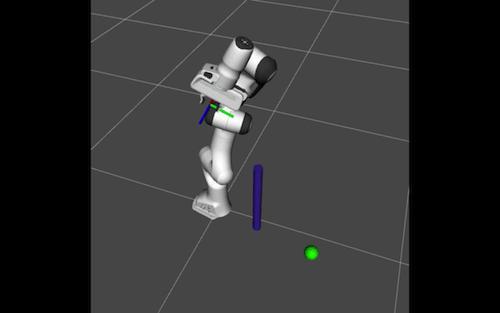}
		\caption{\texttt{learner-0}}
	\end{subfigure}
	\begin{subfigure}[b]{1\linewidth}
		\centering
		\includegraphics[trim={97.5 0 97.5 0},clip,width=0.19\linewidth]{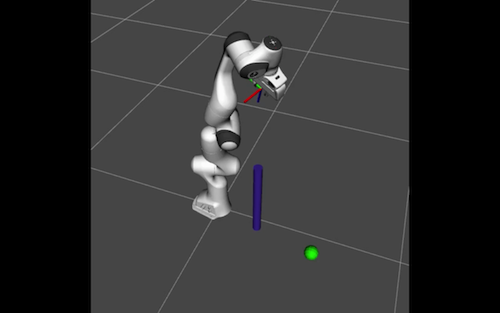}
		\includegraphics[trim={97.5 0 97.5 0},clip,width=0.19\linewidth]{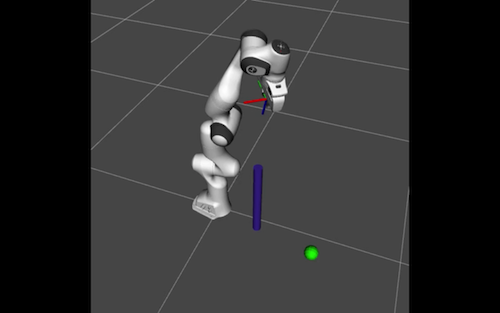}
		\includegraphics[trim={97.5 0 97.5 0},clip,width=0.19\linewidth]{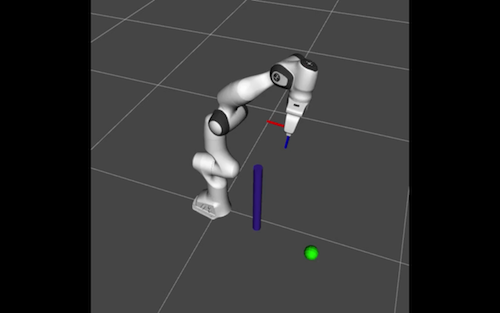}
		\includegraphics[trim={97.5 0 97.5 0},clip,width=0.19\linewidth]{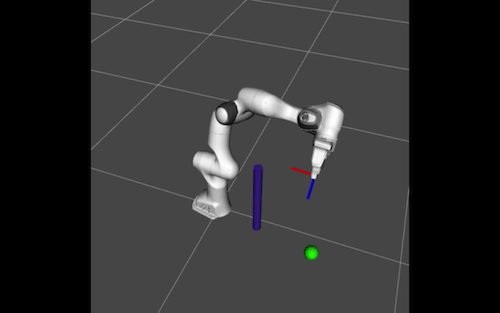}
		\includegraphics[trim={97.5 0 97.5 0},clip,width=0.19\linewidth]{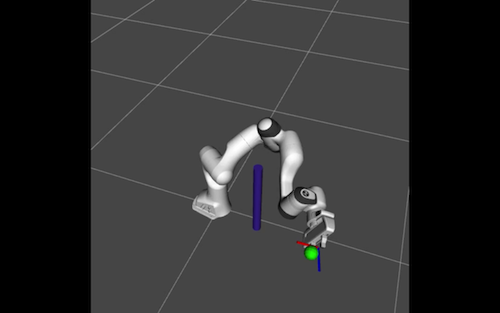}
		\caption{\texttt{learner-300}}
	\end{subfigure}
	\begin{subfigure}[b]{1\linewidth}
		\centering
		\includegraphics[trim={97.5 0 97.5 0},clip,width=0.19\linewidth]{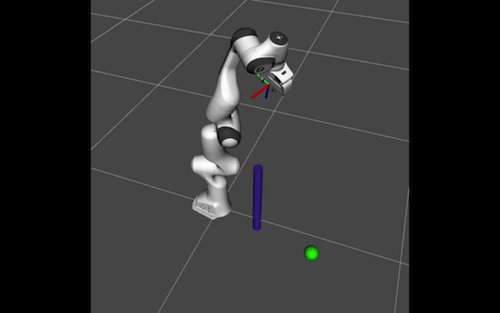}
		\includegraphics[trim={97.5 0 97.5 0},clip,width=0.19\linewidth]{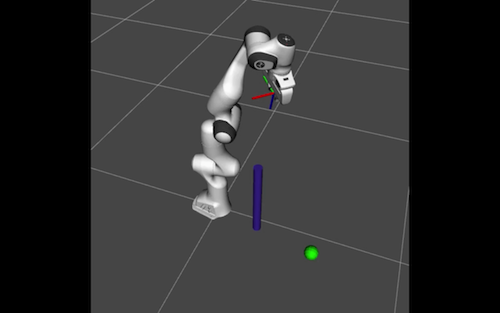}
		\includegraphics[trim={97.5 0 97.5 0},clip,width=0.19\linewidth]{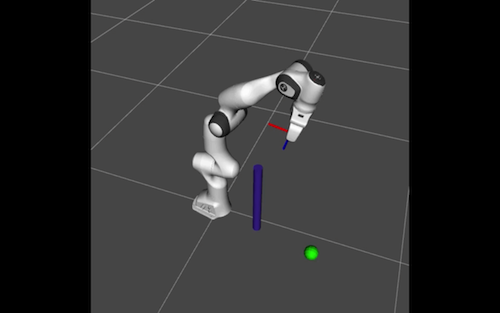}
		\includegraphics[trim={97.5 0 97.5 0},clip,width=0.19\linewidth]{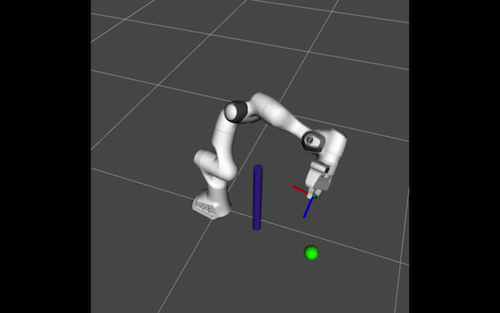}
		\includegraphics[trim={97.5 0 97.5 0},clip,width=0.19\linewidth]{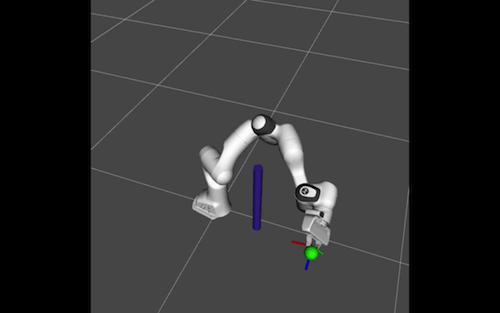}
		\caption{\texttt{learner-1200}}
	\end{subfigure}
	\begin{subfigure}[b]{1\linewidth}
		\centering
		\includegraphics[trim={0 0 0 0},clip,width=0.28\linewidth]{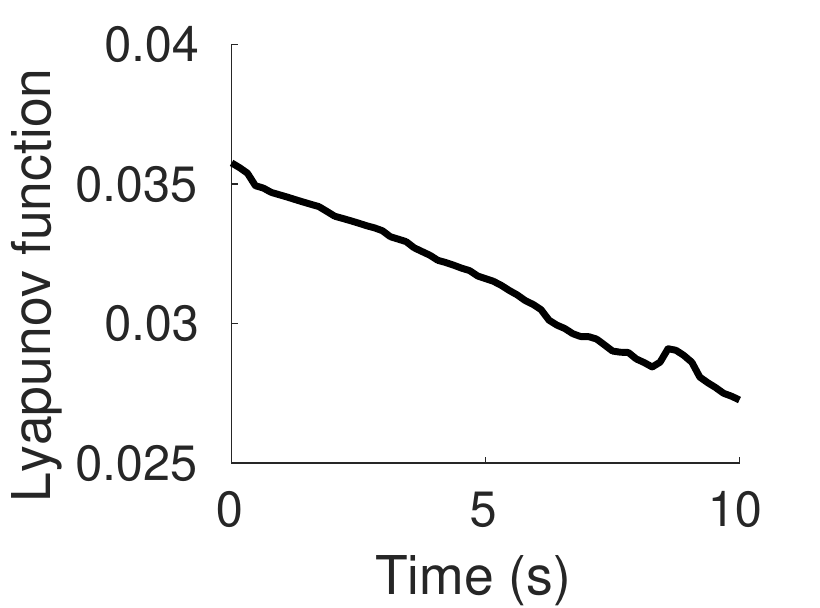}
		\includegraphics[trim={0 0 0 0},clip,width=0.28\linewidth]{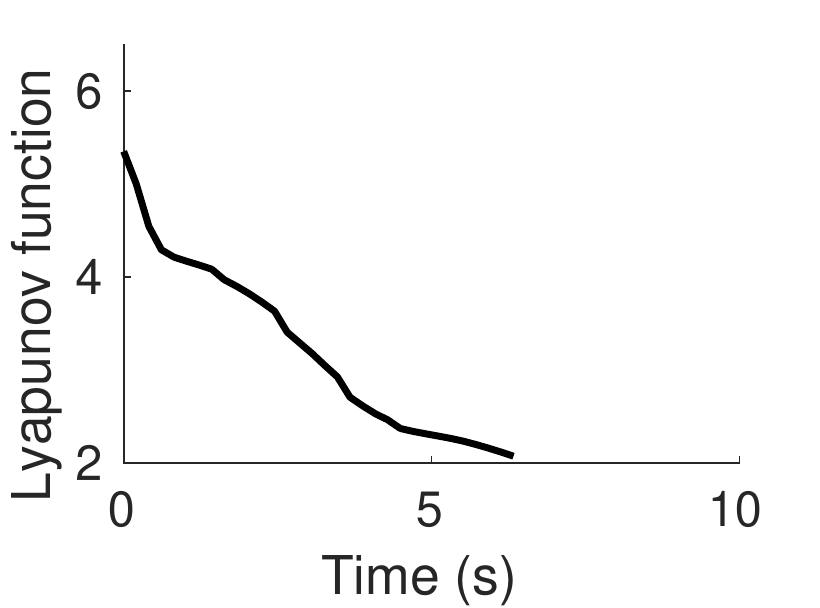}
		\includegraphics[trim={0 0 0 0},clip,width=0.28\linewidth]{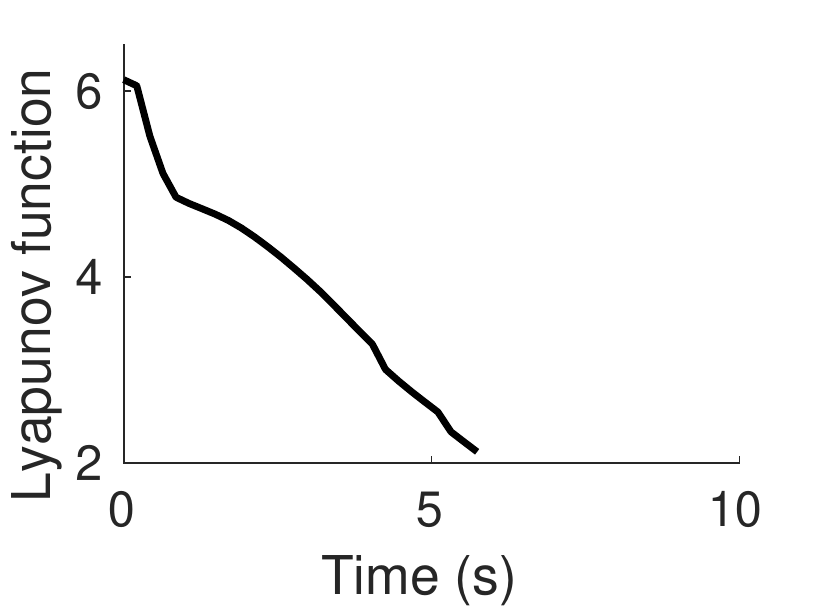}
		\caption{}
	\end{subfigure}
	\caption{\small (a)-(d) An example execution (left to right) from the test dataset, comparing (a) the expert with (b) \texttt{learner-0}, (c) \texttt{learner-300}, and (d) \texttt{learner-1200}. (e) The respective Lyapunov function of the learners' trajectories (\texttt{learner-0} (left), \texttt{learner-300} (middle), \texttt{learner-1200} (right)).}
	\label{fig:franka_traj}
\end{figure}

\subsection{Franka Robot}
From the root node we have various task spaces, like the end-effector position (ee) on which the attractor space (a) is defined by a change of coordinates such that the goal position is at the origin. The attractor RMP (a\textsubscript{rmp}) is then defined on the attractor space for a goal reaching subtask. Each joint of the robot is mapped to a one dimensional upper (ujl\textsubscript{i}) and lower (ljl\textsubscript{i}) joint limit space where a joint limit RMP (jl\textsubscript{rmp}) is defined for joint limit avoidance subtasks. The root node is also mapped to a pre-specified number of control points on the robot (cp\textsubscript{i}) such that they collectively approximate the robot's body and can be used for collision avoidance. On any control point space we add a distance space to the obstacle (d\textsubscript{i}) where the obstacle RMP (o\textsubscript{rmp}) is defined. Note that when multiple obstacles are present we can add distance spaces and the obstacle RMPs for each obstacle on every control point. Now, since the tree structure can change with the number of obstacles, in practice, shared weights can be specified across all obstacles on a given control point, such that training can be performed with only one obstacle to learn the weight function and then can be applied to arbitrary number of obstacles during execution. Finally, there are also native RMPs defined on the root node like a constant damper RMP (q\textsubscript{d}) and an RMP which is just an identity metric (q\textsubscript{mi}) with no learnable weight function to ensure the \resolve operator is numerically stable.

Figure~\ref{fig:franka_traj} shows a qualitative comparison on an example execution with the expert and the learners. We verify the stability properties of \newflow (even during learning) with the monotonically decreasing Lyapunov function plots on these executions. Note that the scale on the plot for \texttt{learner-0} is very small and the tiny kink on the plot is due to numerical issues with Euler integration.

\subsection{Discussion}

The experiments shown here were designed to study if RMPfusion can combine imperfect subtask RMPs, whose inertia weight functions are incorrectly specified while motion policies are sensibly designed with domain knowledge. While this setup does not emulate the full generality where everything is unknown, we think that it captures a representative and important scenario that often happens in practice. We’ve had extensive literature in designing motion policies, whereas designing the associated metrics/inertias for these policies is a fairly new and nontrivial concept, which is a major user burden imposed by RMPflow. 

We address this issue by learning the weight functions, and show in the experiments that imperfect subtask RMPs with poorly designed metrics can still be compensated by our framework. Importantly, we emphasize that RMPfusion is designed for generality and does not assume the knowledge that only the inertias are wrongly specified. Therefore, though not tested in the current experiments, we do believe RMPfusion can be used in more general setups, so long as the user provides sufficiently rich subtask RMPs such that there exists a fusion that can generate the desired behavior. However, how to choose the subtask RMPs to start with is a domain specific problem, similar to specifying the size and structure of a neural network in general. Therefore, we consider it beyond the scope of the current paper, because our main focus here is to study and validate the theoretical benefits of RMPfusion (like stability during immature learning).

Generally, an RMPfusion policy with constant weights (not a function of the parent state, etc.) can be reduced into an RMPflow policy with the same tree structure. This can be seen from \eqref{eq:modifed pullback}; when the weights are constant, we can effectively push all the weights of an RMP-tree* to the leaf-nodes to define modified inertia matrices on an RMP-tree (the motion policy doesn't change). In other words, in the experiments, the expert can be viewed as an RMPflow policy with some unknown inertia matrices and therefore RMPflow wasn't directly compared.

Using neural networks to parameterize the weight functions maybe is an overkill in our experiments. The reason for using general function approximators here is to show that our framework is practically feasible and can support situations where this will become necessary. For example, this allows for learning general differentiable representation for the weight functions, e.g., using images for auxiliary states. However, one should note also that while using expressive function approximators would add representation to the whole policy it could also potentially make learning more difficult.

\end{document}